\documentclass{article}

 \usepackage[preprint]{neurips_2026}

\usepackage[utf8]{inputenc} 
\usepackage[T1]{fontenc}    
\usepackage{hyperref}       
\usepackage{url}            
\usepackage{booktabs}       
\usepackage{amsfonts}       
\usepackage{nicefrac}       
\usepackage{microtype}      
\usepackage{xcolor}         

\usepackage{natbib}
\usepackage{algorithm}
\usepackage{algpseudocode}
\usepackage{paralist}
\usepackage{multirow}

\usepackage{wrapfig}

\usepackage{url}
\usepackage{color}
\usepackage{makeidx}  
\usepackage{xspace}
\usepackage{epstopdf}
\usepackage{mathrsfs}
\usepackage{enumerate}
\usepackage{graphicx,epsfig}
\usepackage{bm}
\usepackage{bbm}
\usepackage{upgreek}
\usepackage{caption}
\usepackage{subcaption}

\usepackage{amsmath}
\usepackage{amssymb}
\usepackage{mathtools}
\usepackage{amsthm}
\usepackage{float}

\usepackage{cleveref}

\title{Arbitrarily Conditioned Hierarchical Flows for Spatiotemporal Events}

\author{%
  Keyan Chen\\
  Kahlert School of Computing\\
  University of Utah\\
  \texttt{u1466725@utah.edu} \\
  \And
  Qiwei Yuan \\
  Kahlert School of Computing\\
  University of Utah\\
  \texttt{joshua.yuan@utah.edu} \\
  \AND
  Zhitong Xu \\
  Kahlert School of Computing\\
  University of Utah\\
  \texttt{u1502956@utah.edu} \\
  \And
  Bin Shen \\
  Celonis AI \\
  \texttt{stanshenbin@gmail.com} \\
  \And
  Shandian Zhe \\
  Kahlert School of Computing\\
  University of Utah\\
  \texttt{zhe@cs.utah.edu} \\
}

\theoremstyle{plain}
\newtheorem{theorem}{Theorem}[section]

\newtheorem{lemma}[theorem]{Lemma}

\theoremstyle{definition}

\theoremstyle{remark}

\newcommand{\ours}{ARCH\xspace}

\begin{document}

\maketitle

\newcommand{\var}{{\rm var}}
\newcommand{\vtrans}[2]{{#1}^{(#2)}}
\newcommand{\kron}{\otimes}
\newcommand{\schur}[2]{({#1} | {#2})}
\newcommand{\schurdet}[2]{\left| ({#1} | {#2}) \right|}
\newcommand{\had}{\circ}
\newcommand{\diag}{{\rm diag}}
\newcommand{\invdiag}{\diag^{-1}}
\newcommand{\rank}{{\rm rank}}
 \newcommand{\expt}[1]{\langle #1 \rangle}
\newcommand{\nullsp}{{\rm null}}
\newcommand{\tr}{{\rm tr}}
\renewcommand{\vec}{{\rm vec}}
\newcommand{\vech}{{\rm vech}}
\renewcommand{\det}[1]{\left| #1 \right|}
\newcommand{\pdet}[1]{\left| #1 \right|_{+}}
\newcommand{\pinv}[1]{#1^{+}}
\newcommand{\erf}{{\rm erf}}
\newcommand{\hypergeom}[2]{{}_{#1}F_{#2}}
\newcommand{\mcal}[1]{\mathcal{#1}}
\newcommand{\bepsilon}{\boldsymbol{\epsilon}}
\newcommand{\brho}{\boldsymbol{\rho}}
\renewcommand{\a}{{\bf a}}
\renewcommand{\b}{{\bf b}}
\renewcommand{\c}{{\bf c}}
\renewcommand{\d}{{\rm d}}  
\newcommand{\e}{{\bf e}}
\newcommand{\f}{{\bf f}}
\newcommand{\g}{{\bf g}}
\newcommand{\h}{{\bf h}}
\newcommand{\bi}{{\bf i}}
\newcommand{\bj}{{\bf j}} 

\renewcommand{\k}{{\bf k}}
\newcommand{\m}{{\bf m}}
\newcommand{\mhat}{{\overline{m}}}
\newcommand{\tm}{{\tilde{m}}}
\newcommand{\n}{{\bf n}}
\renewcommand{\o}{{\bf o}}
\newcommand{\p}{{\bf p}}
\newcommand{\q}{{\bf q}}
\newcommand{\wy}{{\widehat{\y}}}
\newcommand{\wlam}{{\widehat{\lambda}}}
\renewcommand{\r}{{\bf r}}
\newcommand{\s}{{\bf s}}
\renewcommand{\t}{{\bf t}}
\renewcommand{\u}{{\bf u}}
\renewcommand{\v}{{\bf v}}
\newcommand{\w}{{\bf w}}
\newcommand{\x}{{\bf x}}
\newcommand{\y}{{\bf y}}
\newcommand{\z}{{\bf z}}
\newcommand{\A}{{\bf A}}
\newcommand{\B}{{\bf B}}
\newcommand{\C}{{\bf C}}
\newcommand{\D}{{\bf D}}
\newcommand{\F}{{\bf F}}
\newcommand{\G}{{\bf G}}
\newcommand{\Gcal}{{\mathcal{G}}}
\newcommand{\Dcal}{\mathcal{D}}
\newcommand{\Qcal}{{\mathcal{Q}}}
\newcommand{\Pcal}{{\mathcal{P}}}
\newcommand{\Hcal}{{\mathcal{H}}}
\renewcommand{\H}{{\bf H}}
\newcommand{\I}{{\bf I}}
\newcommand{\J}{{\bf J}}
\newcommand{\K}{{\bf K}}
\renewcommand{\L}{{\bf L}}
\newcommand{\Lcal}{{\mathcal{L}}}
\newcommand{\M}{{\bf M}}
\newcommand{\Mcal}{{\mathcal{M}}}
\newcommand{\Ocal}{{\mathcal{O}}}
\newcommand{\Fcal}{{\mathcal{F}}}
\newcommand{\N}{\mathcal{N}}  
\newcommand{\bupeta}{\boldsymbol{\upeta}}
\renewcommand{\O}{{\bf O}}
\renewcommand{\P}{{\bf P}}
\newcommand{\Q}{{\bf Q}}
\renewcommand{\S}{{\bf S}}
\newcommand{\Scal}{{\mathcal{S}}}
\newcommand{\T}{{\bf T}}
\newcommand{\Tcal}{{\mathcal{T}}}
\newcommand{\U}{{\bf U}}
\newcommand{\Ucal}{{\mathcal{U}}}
\newcommand{\tUcal}{{\tilde{\Ucal}}}
\newcommand{\V}{{\bf V}}
\newcommand{\W}{{\bf W}}
\newcommand{\Wcal}{{\mathcal{W}}}
\newcommand{\Vcal}{{\mathcal{V}}}
\newcommand{\X}{{\bf X}}
\newcommand{\Xcal}{{\mathcal{X}}}
\newcommand{\Y}{{\bf Y}}
\newcommand{\Ycal}{{\mathcal{Y}}}
\newcommand{\Z}{{\bf Z}}
\newcommand{\Zcal}{{\mathcal{Z}}}

\newcommand{\bfLambda}{\boldsymbol{\Lambda}}

\newcommand{\bsigma}{\boldsymbol{\sigma}}
\newcommand{\balpha}{\boldsymbol{\alpha}}
\newcommand{\bpsi}{\boldsymbol{\psi}}
\newcommand{\bphi}{\boldsymbol{\phi}}
\newcommand{\bPhi}{\boldsymbol{\Phi}}
\newcommand{\bbeta}{\boldsymbol{\beta}}
\newcommand{\Beta}{\boldsymbol{\eta}}
\newcommand{\btau}{\boldsymbol{\tau}}
\newcommand{\bvarphi}{\boldsymbol{\varphi}}
\newcommand{\bzeta}{\boldsymbol{\zeta}}

\newcommand{\blambda}{\boldsymbol{\lambda}}
\newcommand{\bLambda}{\mathbf{\Lambda}}

\newcommand{\btheta}{\boldsymbol{\theta}}
\newcommand{\bpi}{\boldsymbol{\pi}}
\newcommand{\bxi}{\boldsymbol{\xi}}
\newcommand{\bSigma}{\boldsymbol{\Sigma}}
\newcommand{\bPi}{\boldsymbol{\Pi}}
\newcommand{\bOmega}{\boldsymbol{\Omega}}

\newcommand{\bx}{{\bf x}}
\newcommand{\bgamma}{\boldsymbol{\gamma}}
\newcommand{\bGamma}{\boldsymbol{\Gamma}}
\newcommand{\bUpsilon}{\boldsymbol{\Upsilon}}

\newcommand{\bmu}{\boldsymbol{\mu}}
\newcommand{\0}{{\bf 0}}

\newcommand{\bs}{\backslash}
\newcommand{\ben}{\begin{enumerate}}
\newcommand{\een}{\end{enumerate}}

 \newcommand{\notS}{{\backslash S}}
 \newcommand{\nots}{{\backslash s}}
 \newcommand{\noti}{{\backslash i}}
 \newcommand{\notj}{{\backslash j}}
 \newcommand{\nott}{\backslash t}
 \newcommand{\notone}{{\backslash 1}}
 \newcommand{\nottp}{\backslash t+1}

\newcommand{\notk}{{^{\backslash k}}}
\newcommand{\notij}{{^{\backslash i,j}}}
\newcommand{\notg}{{^{\backslash g}}}
\newcommand{\wnoti}{{_{\w}^{\backslash i}}}
\newcommand{\wnotg}{{_{\w}^{\backslash g}}}
\newcommand{\vnotij}{{_{\v}^{\backslash i,j}}}
\newcommand{\vnotg}{{_{\v}^{\backslash g}}}
\newcommand{\half}{\frac{1}{2}}
\newcommand{\msgb}{m_{t \leftarrow t+1}}
\newcommand{\msgf}{m_{t \rightarrow t+1}}
\newcommand{\msgfp}{m_{t-1 \rightarrow t}}

\newcommand{\proj}[1]{{\rm proj}\negmedspace\left[#1\right]}

\newcommand{\dif}{\mathrm{d}}
\newcommand{\abs}[1]{\lvert#1\rvert}
\newcommand{\norm}[1]{\lVert#1\rVert}

\newcommand{\mrm}[1]{\mathrm{{#1}}}
\newcommand{\RomanCap}[1]{\MakeUppercase{\romannumeral #1}}
\newcommand{\EE}{\mathbb{E}}
\newcommand{\bbI}{\mathbb{I}}
\newcommand{\bbH}{\mathbb{H}}
\newcommand{\ie}{{\textit{i.e.,}}\xspace}
\newcommand{\eg}{{\textit{e.g.,}}\xspace}
\newcommand{\etc}{{\textit{etc.}}\xspace}
\newcommand{\cmt}[1]{}	
\begin{abstract}
 Events in spatiotemporal systems are ubiquitous, yet modeling their complex distributions remains challenging. Existing point process models often rely on strong structural assumptions and are typically limited to autoregressive, event-by-event prediction. As a result, they struggle to support broader inference tasks such as inverse inference, trajectory reconstruction, and recovery of missing event locations. We introduce Arbitrarily Conditioned Hierarchical Flows (\ours), a hierarchical flow matching framework for spatiotemporal event modeling. \ours is expressive enough to capture complex event distributions while enabling tractable and accurate computation of conditional intensities, which quantify instantaneous event risk. 
 Built on a history-encoder-generative-decoder architecture, \ours introduces a hybrid masking strategy for flexible conditioning on arbitrary observed events. This enables a unified treatment of forecasting, inverse inference, and partial trajectory recovery within a single framework. Experiments on synthetic and real-world datasets show that \ours consistently outperforms existing baselines across both prediction and conditional inference tasks.
\end{abstract}
\section{Introduction}
Spatiotemporal events arise in many domains, including weather phenomena, natural disasters, epidemics, and crime. Their distributions are often highly complex, while observations are typically sparse, making learning and inference over such events particularly challenging.

Existing spatiotemporal point process models often rely on strong structural assumptions, such as the independent-increment property of Poisson processes, additive interaction structures in Hawkes-process-like models~\citep{yuan2019multivariate,pmlr-v168-zhou22a,NEURIPS2023_9d30c2de}, and parametric distribution families~\citep{shchur2020intensity}. More recent approaches have explored diffusion-based event generation~\citep{yuan2023DSTPP}, which substantially improves expressive power for capturing complex distributions. However, this flexibility often comes at the cost of losing tractable or well-calibrated conditional intensities, a central quantity for quantifying instantaneous event risk. Furthermore, existing methods are typically limited to autoregressive, event-by-event inference, making them ill-suited for broader tasks such as inverse inference, event trajectory reconstruction, and recovery of missing event locations.

To overcome these bottlenecks, we propose \ours, a hierarchical flow matching framework for spatiotemporal event modeling. Our method is flexible enough to capture complex event distributions from data while also enabling efficient and accurate computation of conditional intensities. In addition, it supports arbitrarily conditioned inference, allowing a broad range of tasks, including forecasting, inverse inference, missing-event imputation, and partial trajectory recovery. The major contributions of our work are summarized as follows:
\begin{compactitem}

    \item \textbf{Model.} We propose a hierarchical flow matching framework that first generates the event time and then, conditioned on the event time, generates the event spatial location. This decomposed generation procedure not only retains the flexibility needed to capture complex distributions, but also enables tractable and accurate computation of the conditional intensity, which quantifies instantaneous event risk. Specifically, we leverage the Picard--Lindel\"of theorem and the order-preserving property of one-dimensional ODEs to compute the conditional intensity of the event-time flow, and then solve the companion log-probability ODE for the event-location flow, combining the two to obtain the final conditional event intensity.
    
    \item \textbf{Hybrid Masked Training.} 
    To support a variety of inference tasks, we employ a history encoder and a generative decoder for each flow. We further introduce a hybrid masking strategy that enables parallel autoregressive generation, randomly conditioned inference, batched trajectory recovery, and future trajectory forecasting. During training, masks are continuously sampled, and the model is trained to recover the target velocity field for masked events conditioned on partially observed events. In this way, a broad range of inference tasks is covered, with training and inference unified within a single framework.
    
    \item \textbf{Experiments.} On three real-world benchmark datasets, \ours achieves strong performance across a broad range of tasks, including next-event prediction, initial-event inference, missing-event imputation, and partial-trajectory recovery. \ours also supports joint multi-step forecasting, achieving performance comparable to or better than its autoregressive variant and diffusion-based autoregressive baselines. On synthetic datasets generated from different spatiotemporal point processes, \ours more effectively recovers the conditional intensity and achieves lower relative $L_2$ errors than competing structured spatiotemporal point process models.
    
\end{compactitem}

\section{Preliminaries}

Flow matching~\citep{lipman2022flow} learns a continuous-time vector field that transports a simple source distribution to a target data distribution through an ordinary differential equation (ODE).
Let $\mathbf{x}_0 \sim p_0$ denote a random sample from a tractable source distribution, typically a standard Gaussian, and let $\mathbf{x}_1 \sim p_{\mathrm{data}}$ denote a data sample. The goal is to construct a time-dependent probability path $\{p_t\}_{t \in [0,1]}$ that smoothly connects $p_0$ and $p_{\mathrm{data}}$, and to learn a velocity field $\mathbf{v}_\theta(\mathbf{x}, t)$ such that samples evolve according to the ODE:
\begin{align}
\frac{d\mathbf{x}(t)}{dt} = \mathbf{v}_\theta(\mathbf{x}(t), t), \qquad t \in [0,1]. \label{eq:flow}
\end{align}
If the learned velocity field matches the true velocity field associated with the probability path, then integrating the ODE from an initial noise sample $\mathbf{x}(0)=\mathbf{x}_0$ transports $\mathbf{x}_0$ to a sample $\mathbf{x}(1)$  distributed according to $p_{\mathrm{data}}$. 
A common choice in flow matching is to define a conditional probability path between paired samples $(\mathbf{x}_0, \mathbf{x}_1)$. One simple and widely used example is the linear flow, which interpolates between source and target as $\mathbf{x}_t = (1-t)\mathbf{x}_0 + t\mathbf{x}_1$ where $t \in [0, 1]$. 
Under this construction, the corresponding conditional velocity is constant along the path: $\mathbf{u}_t(\mathbf{x}_1 \mid \mathbf{x}_0) = \frac{d\mathbf{x}_t}{dt} = \mathbf{x}_1 - \mathbf{x}_0$. 
Flow matching trains $\mathbf{v}_\theta$ to match this target velocity field along sampled intermediate points. In practice, one samples a data point $\mathbf{x}_1 \sim p_{\mathrm{data}}$, a noise sample $\mathbf{x}_0 \sim p_0$, and a time $t \sim \mathrm{Uniform}(0,1)$, constructs the interpolated point $\mathbf{x}_t$, and minimizes the regression loss 
\begin{align}
\mathcal{L}_{\mathrm{FM}}(\theta)=
\mathbb{E}_{t,\mathbf{x}_0,\mathbf{x}_1}
\left[
\left\|
\mathbf{v}_\theta(\mathbf{x}_t, t) - (\mathbf{x}_1 - \mathbf{x}_0)
\right\|^2
\right]. \label{eq:fm-loss}
\end{align}
More generally, other probability paths can be used, with corresponding target velocities derived from the chosen transport construction.

After training, generation is performed by first drawing the initial state  $\mathbf{x}(0) \sim p_0$ from the source distribution, and then solving the learned ODE~\eqref{eq:flow} forward in time. 
The final state $\mathbf{x}(1)$ is taken as a generated sample. In practice, the ODE is solved numerically using a standard solver, such as Euler or higher-order Runge-Kutta methods. Thus, once the velocity field has been learned, sampling reduces to deterministic transport from noise to data through the learned continuous-time dynamics.

\section{Methodology}


Consider a sequence of spatiotemporal events $\Gamma = \left[(s_1, \x_1), \ldots, (s_N, \x_N)\right]$ 
, where $s_1 < \cdots < s_N$, each $s_n \in \mathbb{R}_+$ denotes the timestamp of event $n$, and $\x_n \in \mathbb{R}^d$ denotes its spatial location. In typical applications, $d \leq 3$. Given a collection of observed event sequences $\Dcal$, our goal is to learn the underlying distribution governing the generation of such sequences, while supporting a broad range of inference tasks, including forecasting future events, inverse inference of initial events, imputing missing event times or locations, and reconstructing partial event trajectories.



\subsection{Hierarchical Flow-Matching Framework}

We propose a hierarchical flow-matching framework that flexibly captures complex event distributions while also enabling tractable and accurate computation of conditional intensities, which quantify instantaneous event risk and play a central role in event modeling tasks such as risk monitoring, survival analysis, hotspot discovery, and intervention planning~\cite{cox1972regression,daley2003introduction,braga2001effects}. The overall architecture is illustrated in Appendix Figure~\ref{fig:arch-architecture}.

To illustrate the framework, consider generating a future event $\e = (s, \x)$ given the history sequence $\Hcal = [(s_1, \x_1), \ldots, (s_M, \x_M)].
$
Instead of constructing a single flow that simultaneously generates the event time $s \in \mathbb{R}_+$ and spatial location $\x = (x_1, \ldots, x_d) \in \mathbb{R}^d$, we introduce two flows, $\Gcal_1$ and $\Gcal_2$: $\Gcal_1$ generates $s$ conditioned on $\Hcal$, while $\Gcal_2$ generates $\x$ conditioned on the generated time $s$ and $\Hcal$. This yields the decomposition
\begin{align}
    p(s, \x \mid \Hcal) = p(s \mid \Hcal)\, p(\x \mid s, \Hcal), \label{eq:flow-density-decomp}
\end{align}
where $\Gcal_1$ models $p(s \mid \Hcal)$ and $\Gcal_2$ models $p(\x \mid s, \Hcal)$. Note that $\Gcal_1$ and $\Gcal_2$ are not restricted to autoregressive next-event generation; as explained in Section~\ref{sect:train-infer}, they support arbitrarily conditioned generation more generally.

A key advantage of this design is that it enables efficient and accurate computation of the conditional intensity, i.e., the instantaneous event risk\footnote{More precisely, the conditional intensity, also called the hazard, characterizes the instantaneous likelihood of an event occurring at time $s$, conditioned on the event history and on no new event having occurred prior to $s$.}:
\begin{align}
    \lambda(s, \x \mid \Hcal) = \frac{p(s, \x \mid \Hcal)}{1 - F_S(s \mid \Hcal)}, \label{eq:lam}
\end{align}
where $F_S(\cdot \mid \Hcal)$ denotes the marginal cumulative distribution function (CDF) of the event time:
\begin{align}
    F_S(s \mid \Hcal)
    = p(S \le s \mid \Hcal)
    = \int_{\Omega \cap (s_M, s] \times (-\infty, \infty)^d}
    p(\tau, \x \mid \Hcal)\, d\tau\, \d x_1 \cdots \d x_d,
    \label{eq:cdf-margin}
\end{align}
and $\Omega$ is the spatiotemporal domain of the events.

If one instead uses a single flow to generate $s$ and $\x$ simultaneously, then computing the marginal CDF $F_S(s \mid \Hcal)$ in~\eqref{eq:lam} from the joint density $p(s, \x \mid \Hcal)$ requires integrating over a $(d+1)$-dimensional spatiotemporal domain, as shown in~\eqref{eq:cdf-margin}. This integral typically does not admit a closed form, while numerical integration or Monte Carlo approximation can be computationally expensive and may introduce substantial approximation error. In contrast, our hierarchical flows induce the density decomposition in~\eqref{eq:flow-density-decomp}, which in turn yields: 
\begin{align}
    \lambda(s, \x \mid \Hcal)
    = \lambda(s \mid \Hcal)\, p(\x \mid s, \Hcal),
    \qquad
    \lambda(s \mid \Hcal)
    = \frac{p(s \mid \Hcal)}{1 - F_S(s \mid \Hcal)}.
\end{align}
Thus, the full conditional intensity factors into the marginal temporal intensity $\lambda(s \mid \Hcal)$ and the conditional spatial density $p(\x \mid s, \Hcal)$, which can be computed separately from $\Gcal_1$ and $\Gcal_2$.

We first consider $p(\x \mid s, \Hcal)$. Let the ODE associated with $\Gcal_2$ be
\begin{align}
    \frac{d\x(t)}{dt}
    =
    \v_{\Theta_2}\bigl(\x(t), t \mid \Hcal, s\bigr),
    \label{eq:flow-g2}
\end{align}
where $\v_{\Theta_2}(\cdot)$ denotes the vector field and $\Theta_2$ its parameters. By the instantaneous change-of-variable formula, the log-density along the trajectory $\x(t)$ evolves according to
\begin{align}
    \frac{d}{dt}\log p(\x(t) \mid s, \Hcal)
    =
    - \nabla \cdot \v_{\Theta_2}\bigl(\x(t), t \mid \Hcal, s\bigr),
    \label{eq:logp-g2}
\end{align}
where $\nabla \cdot \v$ denotes the divergence of the velocity field. To compute the conditional density at a given value $\x_{\mathrm{val}}$, we augment the flow~\eqref{eq:flow-g2} with~\eqref{eq:logp-g2} and define a joint state $\z(t) = [\x(t), \ell(t)]$, where $\ell(t)$ accumulates the change in log-density. We then set $\x(1) = \x_{\mathrm{val}}$ and $\ell(1)=0$, and solve the coupled ODEs backward from $t=1$ to $t=0$. This yields the corresponding noise sample $\x(0)$, which follows the prior distribution $p_{\mathrm{prior}}$, together with the accumulated change $\ell(0)=\int_1^0 \frac{d}{dt}\log p(\x(t)\mid s,\Hcal)\,dt$.
The final log-density is therefore $\log p(\x \mid s, \Hcal)=\log p_{\mathrm{prior}}\bigl(\x(0)\bigr) - \ell(0)$.

Next, we consider the marginal temporal intensity $\lambda(s \mid \Hcal)$. Given a specific value $s_{\mathrm{val}}$, we follow a similar procedure 
to obtain the corresponding noise value $s(0)$ and the log-density $\log p(s_{\mathrm{val}} \mid \Hcal)$. The remaining quantity is the marginal CDF $F_S(s_{\mathrm{val}} \mid \Hcal)$, which can be obtained from the CDF of the noise variable at $s(0)$ according to the following lemma.

\begin{lemma}\label{lem}
Let \(S(0)\) be a real-valued random variable, and let \(S(t)\) be the solution at time \(t \ge 0\) to the one-dimensional ordinary differential equation $\frac{ds}{dt} = v(s(t), t)$,
with random initial condition \(S(0)\). Let \(F_t(\cdot)\) denote the cumulative distribution function of \(S(t)\). Then, for any deterministic solution trajectory \(\{s(t)\}_{t \ge 0}\) of the ODE with initial value \(s(0)\), the CDF evaluated along the trajectory is invariant in time:
\[
F_t(s(t)) = F_0(s(0)), \qquad \forall t \ge 0.
\]
\end{lemma}
This can be established  using the Picard--Lindel\"of theorem  together with the order-preserving property of one-dimensional ODEs; we defer the proof to Appendix~\ref{sect:appendix:proof}. Since the noise distribution is chosen to be simple, \eg a standard Gaussian, its CDF can be evaluated efficiently. In practice, to ensure accurate ODE solutions, we can use higher-order Runge-Kutta methods, such as DOPRI5~\citep{dormand1980family}.

\subsection{Arbitrarily Conditioned Training and Inference}\label{sect:train-infer}

To support a broad range of inference tasks beyond standard autoregressive, event-by-event prediction, we introduce a hybrid masking strategy for arbitrarily conditioned training and inference.
Specifically, each flow employs a history encoder and a generative decoder, both implemented as Transformers. During training, given an event sequence $\Gamma = [(s_1, \x_1), \ldots, (s_N, \x_N)]$, the encoder treats each event $e_n = (s_n, \x_n) \in \Gamma$ as an input token and produces event representations through self-attention layers. The encoder serves as the conditioning branch, providing integrated feature representations of the observed events. These representations are then incorporated into the decoder through cross-attention to guide generation of the target events. The decoder also contains $N$ tokens, each corresponding to one generated event. The input to decoder token $n$ consists of the flow time $t$ and the intermediate ODE state at time $t$. For the flow $\Gcal_2$ that models spatial locations, \ie $p(\x \mid s, \Hcal)$, the decoder input at token $n$ additionally includes the time $s_n$ of event $n$.

To flexibly support diverse inference tasks, we combine several masking strategies during training:
    \noindent\textbf{Autoregressive Mask.} This mask supports standard autoregressive generation. We impose a causal self-attention mask on the encoder so that each event token $i$ attends only to itself and past events $j < i$. On the decoder side, token $n$ attends to encoder tokens $1, \ldots, n-1$. This cross-attention pattern allows decoder token $n$ to generate event $n$ conditioned on its entire preceding history. Combined with causal self-attention in the encoder, this design enables parallel autoregressive training over all events in the sequence.

    \noindent\textbf{Random Condition Mask.} This mask supports generation of an arbitrary subset of events conditioned on the remaining observed events. Specifically, we sample a binary condition mask $\c = (c_1, \ldots, c_N)$, where each $c_n \in \{0,1\}$, $c_n = 1$ indicates that event $n$ is conditioned (observed), and $c_n = 0$ indicates that event $n$ is to be generated. On the encoder side, we mask out all tokens with $c_n = 0$ and perform full self-attention among the conditioned events to exchange and fuse their representations. On the decoder side, we mask out all tokens with $c_n = 1$ and perform self-attention among the events to be generated. We then set the cross-attention mask $\Mcal[n,m] = 0$ for all pairs $(n,m)$ such that $c_n = 0$ and $c_m = 1$, allowing every generated token to attend to every conditioned token. In this way, the full conditioned subset facilitates joint generation of the target events.

    \noindent\textbf{Consecutive Mask.} This mask supports generation of a consecutive subsequence of events conditioned on the remaining events, which is useful for trajectory reconstruction and joint multi-step forecasting. Although the random condition mask can in principle cover such cases, its independently sampled entries more often produce scattered missing events rather than contiguous missing segments. To address this, we randomly sample two endpoints $a,b$ with $1 \le a < b \le N$, set $c_n = 0$ for $a \le n \le b$, and set $c_n = 1$ otherwise. This yields a mask in which a consecutive block of events is generated conditioned on the events outside the block.


During training, we sample a mini-batch of event sequences and apply each masking strategy to compute the corresponding flow-matching loss.
The temporal flow is trained to generate log-transformed inter-event times rather than absolute timestamps, which removes dependence on the arbitrary time origin, compresses large temporal gaps, and improves numerical stability.
The objective is to recover the target velocity field on the tokens corresponding to events that need to be generated. For the random-conditioning and consecutive masks, the mask pattern is sampled independently for each sequence in the batch. We then sum the resulting losses across all mask types and backpropagate through the full objective to update the model parameters.  The training/generation procedures are summarized in Appendix Algorithms~\ref{alg:arch-training}--\ref{alg:arch-generation}. The masking strategies are illustrated in Figure~\ref{fig:mask-strategy}.

\section{Related Work}\label{sec:related}

\textbf{Temporal Event Modeling.}
A large body of work has developed temporal point process models for temporal event data. Early approaches include Poisson processes and Hawkes processes (HPs)~\citep{hawkes1971spectra,blundell2012modelling,lloyd2015variational,wang2017predicting,zhe2018stochastic,wang2022nonparametric}. Later works moved toward learning conditional intensity functions with deep neural networks. Representative examples include Neural Hawkes Processes (NHP)~\citep{mei2017neural}, which encode event histories using LSTMs~\citep{hochreiter1997long}, Recurrent Marked Temporal Point Processes (RMTPP)~\citep{du2016recurrent}, which similarly use recurrent architectures while explicitly modeling event marks, and Transformer-based approaches~\citep{zhang2020self,zuo2020transformer}.

More recent research has explored modern generative models for temporal event generation. \citet{lin2022exploring} proposed an autoregressive diffusion approach that generates events one at a time. \citet{ludke2023add} introduced a diffusion framework that uses a homogeneous Poisson process as the noise distribution and applies thinning and deletion operations to edit sequences during generation. \citet{ludke2026editbased} further developed edit-based generation by introducing several sequence edit operations to align noise sequences with event sequences of potentially different lengths, and used a flow model to learn the corresponding edit-rate table. \citet{kerrigan2026eventflow} first generates the sequence length and then generates the event sequence accordingly. Despite their expressive power, these methods generally do not provide tractable or well-calibrated conditional intensities, making it difficult to recover the instantaneous event risk that is central to standard event modeling.

\textbf{Spatiotemporal Event Modeling.}
Classical spatiotemporal point process models~\citep{yuan2019multivariate,reinhart2018review} largely extend Poisson and Hawkes processes to the spatiotemporal setting. More recently, \citet{chen2021neuralstpp} proposed a neural spatiotemporal point process that models jumps in the conditional intensity upon event arrivals using recurrent neural networks, and employs neural ordinary differential equations (ODEs)~\citep{chen2018neural} to capture continuous-time intensity evolution between events. In contrast, \citet{jia2019neural} used neural stochastic differential equations (SDEs) to model latent state dynamics governing the conditional intensity.
Other recent works retain the additive, excitation-only structure of spatiotemporal Hawkes processes while increasing modeling flexibility. The model of \citet{pmlr-v168-zhou22a} parameterizes the triggering kernel using Transformer encodings of historical events. \citet{NEURIPS2023_9d30c2de} introduced a neural triggering kernel based on monotonic networks~\citep{sill1997monotonic}, which allows exact likelihood integration. More recently, \citet{yuan2023DSTPP} proposed an autoregressive diffusion model that sequentially generates spatiotemporal events. Although this method demonstrates strong performance on next-event prediction, it, like the temporal generative models above, does not explicitly provide calibrated conditional intensities. Overall, existing spatiotemporal event models either rely on restrictive structural assumptions or prioritize expressive generation without retaining tractable conditional intensities, and are typically focused on autoregressive next-event prediction rather than more general conditioned inference tasks.

\section{Experiments}
\subsection{Intensity Recovery}
\label{sect:intensity}

We first evaluated whether \ours can recover the conditional intensity of event sequences, i.e., the instantaneous event risk. To this end, we constructed two synthetic spatiotemporal point processes. The first is a spatiotemporal Hawkes process with conditional intensity
\begin{align}
\lambda(s,\x|\Hcal_s) = \lambda_0 + \sum\nolimits_{s_n<s} \alpha e^{-\beta(s-s_n)}
\frac{1}{2\pi \sigma^2}
e^{-\frac{\|\x - \x_n\|^2}{2\sigma^2}},
\label{eq:sthp}  
\end{align}
where $s \in [0, 60]$ denotes time, $\x \in [0,1]^2$ denotes the spatial location, and $\Hcal_s = \{(s_n,\x_n)\mid s_n < s\}$ denotes the history of events before time $s$. We set $\lambda_0 = 1.0$, $\alpha=0.72$, $\beta=1.2$, and $\sigma=0.05$. 

The second process is a spatiotemporal self-correcting process,
\begin{align}
    \lambda(s, \x|\Hcal_s) = \mu \exp\left(\alpha s \phi(\x| c_0) - \sum\nolimits_{s_n<s}\beta \phi(\x_n - \x|c_1)\right),
    \label{eq:scstpp}
\end{align}
where $\phi(\cdot | c) = \N(\cdot |\0, c\I)$. We set $\mu=1.0$, $\alpha=0.2$, $\beta = 0.4$, $c_0=0.25$, and $c_1 = 0.2$. The spatiotemporal domain was set to $(s,\x) \in [0,100]\times [0,1]^2$. We refer to the first process as \textbf{STHP-SYN} and the second as \textbf{STSC-SYN}. For each process, we simulated 10K event sequences for training, 1K for validation, and 200 for testing using Ogata's thinning algorithm~\citep{ogata1981lewis}.
\begin{wraptable}{r}{0.5\textwidth}
\centering
\small
\setlength{\tabcolsep}{6pt}
\renewcommand{\arraystretch}{1.05}
\caption{\small Relative $L_2$ error of the conditional intensity on synthetic datasets. Best results are shown in bold.}
\label{tab:synthetic-rel-l2}
\small
\begin{tabular}{lcc}
\toprule
Model & \textbf{STHP-SYN} & \textbf{STSC-SYN} \\
\midrule
AutoSTPP & $0.572 \pm 0.187$ & $0.588 \pm 0.153$ \\
DeepSTPP & $0.631 \pm 0.126$ & $0.918 \pm 0.037$ \\
NSTPP    & $0.981 \pm 0.022$ & $0.819 \pm 0.087$ \\
ARCH     & $\mathbf{0.458 \pm 0.171}$ & $\mathbf{0.370 \pm 0.132}$ \\
\bottomrule
\end{tabular}
\end{wraptable}
We compared \ours with several state-of-the-art spatiotemporal point process models that explicitly parameterize the conditional intensity function. The first baseline is Deep Spatiotemporal Hawkes Process (DeepSTPP)~\citep{pmlr-v168-zhou22a}, which adopts a Hawkes-style intensity with exponential triggering kernels, where the kernel parameters are predicted by a Transformer encoder over past events. The second is Automatic Integration for Spatiotemporal Neural Point Processes (AutoSTPP)~\citep{NEURIPS2023_9d30c2de}, which extends Hawkes-style intensities with neural triggering kernels parameterized by monotonic networks~\citep{sill1997monotonic}, enabling exact likelihood integration. The third is Neural Spatiotemporal Point Process (NSTPP)~\citep{chen2021neuralstpp}, which models the conditional intensity using neural ODEs and updates the event-history representation through GRU-based state transitions when new events arrive. Our method was implemented in PyTorch and trained using the AdamW optimizer~\citep{loshchilov2017decoupled} with a learning rate of $10^{-3}$ for up to 300 epochs. The mini-batch size was set to 64. The two flows in our model share the same Transformer backbone but use separate heads for velocity prediction. For conditional intensity evaluation, we solved the associated ODEs using DOPRI5~\citep{dormand1980family} with a fixed integration budget of 20 solver steps.
Additional training details and hyperparameter settings are provided in the Appendix Section~\ref{sect:appendix:hyper}. For DeepSTPP, AutoSTPP, and NSTPP, we used the official implementations and default hyperparameter settings released by the original authors.

For each test sequence, we evaluated the conditional intensity at every event time. At each time point, we computed an intensity snapshot on a uniform $8 \times 8$ spatial grid. We then computed the relative $L_2$ error of the predicted snapshots against the ground-truth snapshots, aggregated over all evaluated time points. As shown in Table~\ref{tab:synthetic-rel-l2}, \ours achieves substantially smaller relative $L_2$ errors than the competing approaches, demonstrating more accurate recovery of the conditional intensity.
\begin{figure*}[t]
    \centering
    \includegraphics[width=\textwidth]{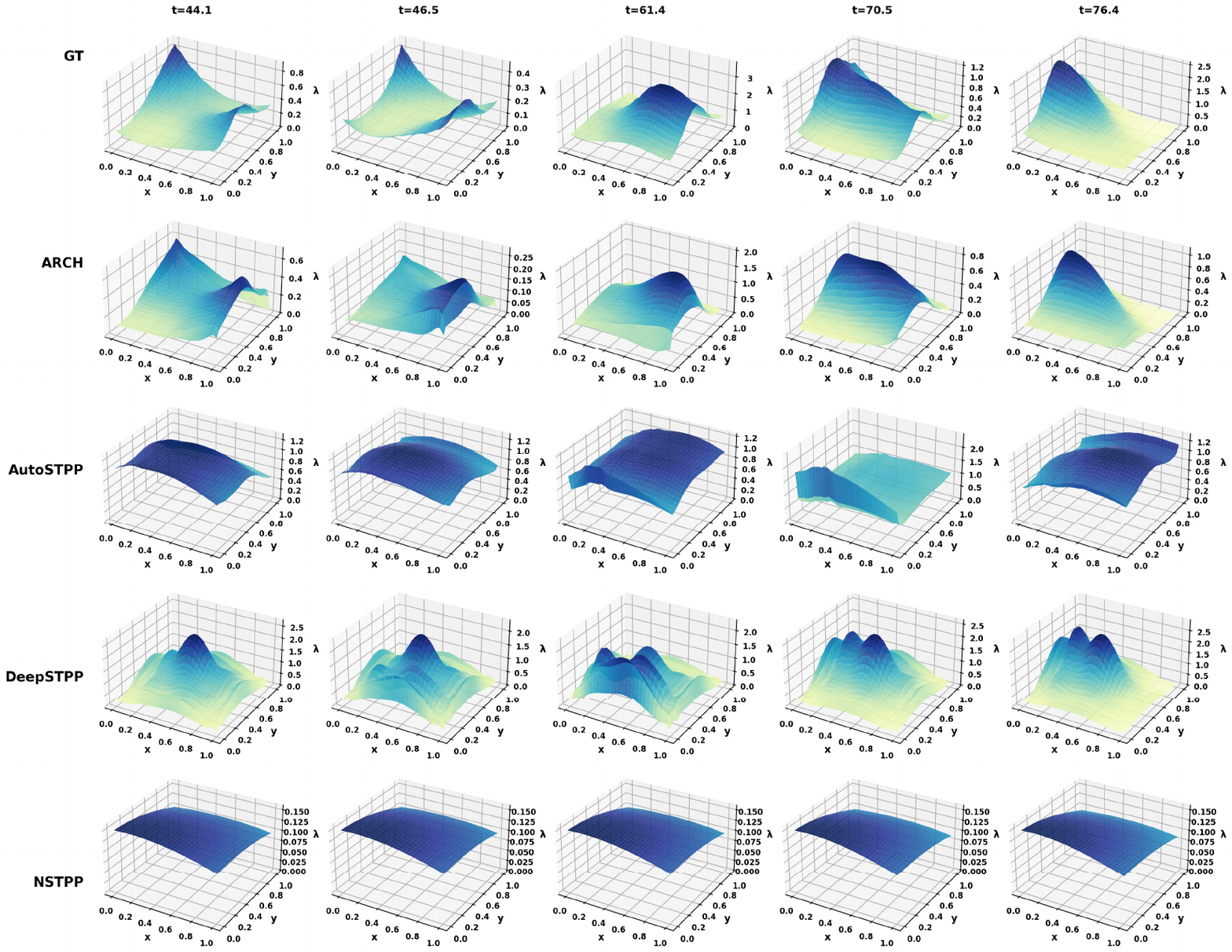}
\caption{\small Conditional intensity functions at representative time points on \textbf{STSC-SYN}. GT: ground-truth.}
    \label{fig:cif-self-correcting}
\end{figure*}

We then randomly selected one test sequence from \textbf{STHP-SYN} and \textbf{STSC-SYN}, and visualized the learned intensity functions at five representative time points. As shown in Figure~\ref{fig:cif-self-correcting} and Appendix Figure~\ref{fig:cif-HP}, \ours closely recovers the shape of the ground-truth intensity across all time points. DeepSTPP and AutoSTPP also capture the overall intensity shape on \textbf{STHP-SYN} reasonably well; see Appendix Figure~\ref{fig:cif-HP}. However, their learned intensities deviate substantially from the ground truth on \textbf{STSC-SYN}. This is likely because both methods impose an additive, excitation-only Hawkes-style intensity structure, whereas \textbf{STSC-SYN} is generated from a multiplicative intensity model with strong inhibitory interactions among events. The mismatch can therefore be attributed, at least in part, to model misspecification. Moreover, even on \textbf{STHP-SYN}, where their recovered intensity shapes are visually close to the ground-truth, DeepSTPP and AutoSTPP exhibit more pronounced scale discrepancies than \ours, which is consistent with their larger relative $L_2$ errors in Table~\ref{tab:synthetic-rel-l2}.
In contrast, \ours faithfully captures the intensity structure on both datasets, demonstrating its flexibility in capturing both excitation-dominated and inhibition-driven spatiotemporal dynamics without relying on a restrictive Hawkes-style parameterization. Finally, NSTPP consistently exhibits poor recovery performance: its learned intensity surfaces show only minor variation across time points. This suggests that NSTPP struggles to learn a time-sensitive conditional spatial distribution in these synthetic settings.

\subsection{Predictive Performance}


We next evaluated \ours on a range of prediction tasks using three real-world benchmark datasets: \textbf{Earthquake}~\citep{chen2021neuralstpp}, \textbf{COVID-19}~\citep{chen2021neuralstpp}, and \textbf{Citibike}~\citep{10.1145/3580305.3599511}. These datasets capture seismic activity, infectious disease propagation, and urban mobility, providing diverse spatiotemporal event scenarios. Details of the datasets are provided in Appendix~\ref{sect:dataset}.
\begin{table*}[th]
\centering
\small 
\caption{\small Root-mean-square error (RMSE) and Euclidean distance for next-event time and location prediction. The best two results are highlighted in bold.}\label{tbl:next-event-pred}
\small 
\begin{tabular}{c|cc|cc|cc}
\hline
& \multicolumn{2}{c|}{Earthquake} & \multicolumn{2}{c|}{COVID-19} & \multicolumn{2}{c}{Citibike} \\ \cline{2-7} 
Model            & Spatial $\downarrow$   & Temporal $\downarrow$  & Spatial $\downarrow$   & Temporal $\downarrow$  & Spatial $\downarrow$   & Temporal  $\downarrow$ \\ \hline

HPP  & 9.45{\scriptsize $\pm$0.000} & 0.412{\scriptsize $\pm$0.000} & 0.818{\scriptsize $\pm$0.000} & 0.113{\scriptsize $\pm$0.000} & 0.452{\scriptsize $\pm$0.000} & 0.239{\scriptsize $\pm$0.000} \\
STHP & 8.35{\scriptsize $\pm$0.252} & 0.424{\scriptsize $\pm$0.018} & 0.422{\scriptsize $\pm$0.000} & {0.100}{\scriptsize $\pm$0.001} & 0.032{\scriptsize $\pm$0.000} & 0.633{\scriptsize $\pm$0.126} \\
NJSDE   &    9.98{\scriptsize $\pm$0.024} &   0.465{\scriptsize $\pm$0.009} &  0.641{\scriptsize $\pm$0.009}  &  0.137{\scriptsize $\pm$0.001} &   0.707{\scriptsize $\pm$0.001}    &  0.264{\scriptsize $\pm$0.005} \\ 
NSTPP      &   8.11{\scriptsize $\pm$0.000}        &  0.547{\scriptsize $\pm$0.010}  &   0.560{\scriptsize $\pm$0.000}   &   0.145{\scriptsize $\pm$0.002}   &    0.705{\scriptsize $\pm$0.000}&  0.355{\scriptsize $\pm$0.013}  \\ 
DeepSTPP        & 9.20{\scriptsize $\pm$0.000}   &   \textbf{0.341{\scriptsize $\pm$0.000}}  &  0.687{\scriptsize $\pm$0.000} &  0.197{\scriptsize $\pm$0.000}     &  0.044{\scriptsize $\pm$0.000}    &   0.234{\scriptsize $\pm$0.000} \\
DSTPP   &   \textbf{6.77{\scriptsize $\pm$0.193}}  &  \textbf{0.375{\scriptsize $\pm$0.001}}    & \textbf{0.419{\scriptsize $\pm$0.001}} &        \textbf{0.093{\scriptsize $\pm$0.000}}  & \textbf{0.031{\scriptsize $\pm$0.000}}   &  \textbf{0.200{\scriptsize $\pm$0.002}}\\ \hline 
\ours (Ours) & \textbf{6.55{\scriptsize $\pm$0.022}} & \textbf{0.375{\scriptsize $\pm$0.001}} & \textbf{0.391{\scriptsize $\pm$0.000}} & \textbf{0.097{\scriptsize $\pm$0.000}} & \textbf{0.031{\scriptsize $\pm$0.000}} & \textbf{0.211{\scriptsize $\pm$0.000}}\\ \hline 
\end{tabular}
\vspace{-0.1in}
\end{table*}

\paragraph{Next Event Prediction.}
We first evaluated \ours on next-event prediction, the standard forecasting task in spatiotemporal event modeling. We compared against several classical and recent state-of-the-art methods: 
(1) DeepSTPP~\citep{pmlr-v168-zhou22a}; 
(2) NSTPP~\citep{chen2021neuralstpp}; 
(3) Diffusion Spatiotemporal Point Process (DSTPP)~\citep{10.1145/3580305.3599511}, a recent diffusion-based generative model that autoregressively generates the next spatiotemporal event conditioned on the history. Due to its DDPM formulation~\citep{ho2020denoising}, DSTPP cannot evaluate the conditional intensity; 
(4) Neural Jump Stochastic Differential Equations (NJSDE)~\citep{jia2019neural}, which encode event histories as latent states and model their dynamics with neural SDEs; 
(5) Spatiotemporal Hawkes Process (STHP), using the conditional intensity form in Eq.~\eqref{eq:sthp}; and 
(6) Homogeneous Poisson Process (HPP). 
We implemented STHP in PyTorch and trained it with the AdamW optimizer using a learning rate of $10^{-3}$ for up to 1K epochs to ensure convergence. HPP admits a closed-form solution and does not require explicit training. The architecture of \ours is kept fixed across all prediction tasks. For next-event prediction, \ours uses the autoregressive mask and performs Euler integration with 10 steps. For all competing neural baselines, we used the public implementations released by the original authors.

We used the same training, validation, and test splits as provided in~\citep{10.1145/3580305.3599511} for all three datasets. Following the experimental protocol of~\citep{chen2021neuralstpp,yuan2023DSTPP}, we repeated each experiment five times with different random initializations. We evaluated temporal prediction using the root-mean-square error (RMSE) of the predicted event time, and spatial prediction using the Euclidean distance between the predicted and ground-truth event locations. We report the mean and standard deviation in Table~\ref{tbl:next-event-pred}. \ours consistently achieves top performance in both temporal and spatial prediction. Its prediction accuracy is often close to that of the diffusion-based model DSTPP, while substantially outperforming other neural point process baselines in many cases. In particular, \ours achieves the best spatial prediction accuracy and the second-best temporal prediction accuracy across all three datasets. Although \ours and DSTPP obtain comparable next-event prediction performance, DSTPP cannot provide calibrated estimates of event risk through the conditional intensity and is restricted to autoregressive next-event generation. 

\begin{table*}[th]
\centering
\small
\caption{\small Inverse inference of the initial events. KNNImp is short for KNNImpute.}
\label{tb:inverse-prediction}
\small
\begin{tabular}{c|cc|cc|cc}
\hline
& \multicolumn{2}{c|}{Earthquake} & \multicolumn{2}{c|}{COVID-19} & \multicolumn{2}{c}{Citibike} \\ \cline{2-7}
Model & Spatial $\downarrow$ & Temporal $\downarrow$ & Spatial $\downarrow$ & Temporal $\downarrow$ & Spatial $\downarrow$ & Temporal $\downarrow$ \\ \hline

\multicolumn{7}{c}{Recovering the first event} \\ \hline
KNNImp ($k=3$)  & 11.8{\scriptsize $\pm$0.000} & 0.386{\scriptsize $\pm$0.000} & 0.927{\scriptsize $\pm$0.000} & 0.133{\scriptsize $\pm$0.000} & 0.052{\scriptsize $\pm$0.000} & 0.363{\scriptsize $\pm$0.000} \\
KNNImp ($k=5$)  & 11.0{\scriptsize $\pm$0.000} & 0.414{\scriptsize $\pm$0.000} & 0.800{\scriptsize $\pm$0.000} & 0.135{\scriptsize $\pm$0.000} & 0.046{\scriptsize $\pm$0.000} & 0.336{\scriptsize $\pm$0.000} \\
KNNImp ($k=10$) & 10.1{\scriptsize $\pm$0.000} & 0.406{\scriptsize $\pm$0.000} & 0.710{\scriptsize $\pm$0.000} & 0.142{\scriptsize $\pm$0.000} & 0.042{\scriptsize $\pm$0.000} & 0.299{\scriptsize $\pm$0.000} \\ 
\ours  & \textbf{7.34{\scriptsize $\pm$0.021}} & \textbf{0.304{\scriptsize $\pm$0.019}} & \textbf{0.478{\scriptsize $\pm$0.002}} & \textbf{0.087{\scriptsize $\pm$0.010}} & \textbf{0.033{\scriptsize $\pm$0.000}} & \textbf{0.070{\scriptsize $\pm$0.033}} \\ \hline

\multicolumn{7}{c}{Recovering the first and second events} \\ \hline
KNNImp ($k=3$)  & 12.2{\scriptsize $\pm$0.000} & 0.506{\scriptsize $\pm$0.000} & 0.972{\scriptsize $\pm$0.000} & 0.174{\scriptsize $\pm$0.000} & 0.052{\scriptsize $\pm$0.000} & 0.462{\scriptsize $\pm$0.000} \\
KNNImp ($k=5$)  & 11.3{\scriptsize $\pm$0.000} & 0.518{\scriptsize $\pm$0.000} & 0.830{\scriptsize $\pm$0.000} & 0.172{\scriptsize $\pm$0.000} & 0.046{\scriptsize $\pm$0.000} & 0.456{\scriptsize $\pm$0.000} \\
KNNImp ($k=10$) & 10.2{\scriptsize $\pm$0.000} & 0.513{\scriptsize $\pm$0.000} & 0.723{\scriptsize $\pm$0.000} & 0.173{\scriptsize $\pm$0.000} & 0.042{\scriptsize $\pm$0.000} & 0.456{\scriptsize $\pm$0.000} \\ 
\ours  & \textbf{7.60{\scriptsize $\pm$0.022}} & \textbf{0.471{\scriptsize $\pm$0.006}} & \textbf{0.464{\scriptsize $\pm$0.001}} & \textbf{0.148{\scriptsize $\pm$0.002}} & \textbf{0.033{\scriptsize $\pm$0.000}} & \textbf{0.373{\scriptsize $\pm$0.011}} \\ \hline
\end{tabular}
\vspace{-0.1in}
\end{table*}

\paragraph{Inverse Inference.}
Second, we evaluated \ours on inverse inference, where the goal is to infer the initial events from subsequent observations. We considered two settings: recovering the first event and recovering the first two events. Since existing spatiotemporal event models are not designed to support such inverse inference tasks, we compared against KNNImpute~\citep{troyanskaya2001missing}, a general-purpose and widely used missing-value imputation method based on nearest neighbors. We used the implementation from Scikit-Learn~\citep{scikit-learn} and varied the number of neighbors $k \in \{3, 5,10\}$.
As shown in Table~\ref{tb:inverse-prediction}, \ours consistently outperforms KNNImpute by a large margin, demonstrating substantially better recovery of the initial events. These results highlight a key advantage of our arbitrarily conditioned formulation: the same trained model can be directly applied to nonstandard inference tasks beyond autoregressive forecasting, without task-specific retraining or architectural modification.

\paragraph{Missing Events Recovery.}
Third, we evaluated \ours on recovering randomly missing events. For each test sequence, we randomly selected a subset of events as missing and applied \ours to infer them conditioned on the remaining observed events. We varied the missing ratio among $10\%$, $20\%$, and $30\%$. As shown in Appendix Table~\ref{tb:completion}, \ours consistently outperforms KNNImpute across nearly all settings. The only exception is temporal prediction on the COVID-19 dataset, where \ours is slightly worse. These results show that, although \ours is trained as a generative model, it can also serve as an effective event imputation model and achieves competitive or superior performance compared with the established KNN-based imputation baseline. We further evaluated \ours on recovering events with partially missing attributes, where the event time or location is unobserved. \ours outperforms KNNImpute in almost all cases; details are provided in Appendix~\ref{sect:missing}.

\begin{table*}[th]
\centering
\small
\caption{\small Internal event trajectory recovery.}
\label{tab:partial-recovery}
\small
\begin{tabular}{c|cc|cc|cc}
\hline
& \multicolumn{2}{c|}{Earthquake} & \multicolumn{2}{c|}{COVID-19} & \multicolumn{2}{c}{Citibike} \\ \cline{2-7}
Model & Spatial $\downarrow$ & Temporal $\downarrow$ & Spatial $\downarrow$ & Temporal $\downarrow$ & Spatial $\downarrow$ & Temporal $\downarrow$ \\ \hline

\multicolumn{7}{c}{Recovery length = 5} \\ \hline
KNNImp ($k=3$)  & 8.71{\scriptsize $\pm$0.123} & 0.558{\scriptsize $\pm$0.011} & 0.540{\scriptsize $\pm$0.007} & 0.186{\scriptsize $\pm$0.006} & 0.037{\scriptsize $\pm$0.000} & 0.193{\scriptsize $\pm$0.006} \\
KNNImp ($k=5$)  & 8.42{\scriptsize $\pm$0.125} & 0.523{\scriptsize $\pm$0.008} & 0.518{\scriptsize $\pm$0.006} & 0.179{\scriptsize $\pm$0.005} & 0.035{\scriptsize $\pm$0.000} & 0.182{\scriptsize $\pm$0.005} \\
KNNImp ($k=10$) & 8.24{\scriptsize $\pm$0.104} & 0.496{\scriptsize $\pm$0.006} & 0.495{\scriptsize $\pm$0.005} & 0.174{\scriptsize $\pm$0.006} & 0.034{\scriptsize $\pm$0.000} & \textbf{0.174{\scriptsize $\pm$0.005}} \\ 
\ours & \textbf{7.90{\scriptsize $\pm$0.127}} & \textbf{0.456{\scriptsize $\pm$0.010}} & \textbf{0.467{\scriptsize $\pm$0.007}} & \textbf{0.173{\scriptsize $\pm$0.008}} & \textbf{0.031{\scriptsize $\pm$0.000}} & 0.176{\scriptsize $\pm$0.005} \\ \hline

\multicolumn{7}{c}{Recovery length = 10} \\ \hline
KNNImp($k=3$)  & 8.91{\scriptsize $\pm$0.060} & 0.548{\scriptsize $\pm$0.008} & 0.536{\scriptsize $\pm$0.006} & 0.214{\scriptsize $\pm$0.005} & 0.037{\scriptsize $\pm$0.000} & 0.196{\scriptsize $\pm$0.002} \\
KNNImp($k=5$)  & 8.53{\scriptsize $\pm$0.058} & 0.518{\scriptsize $\pm$0.006} & 0.510{\scriptsize $\pm$0.006} & 0.204{\scriptsize $\pm$0.004} & 0.035{\scriptsize $\pm$0.000} & 0.182{\scriptsize $\pm$0.003} \\
KNNImp($k=10$) & 8.25{\scriptsize $\pm$0.045} & 0.494{\scriptsize $\pm$0.006} & 0.490{\scriptsize $\pm$0.006} & 0.198{\scriptsize $\pm$0.004} & 0.034{\scriptsize $\pm$0.000} & 0.169{\scriptsize $\pm$0.003} \\
\ours & \textbf{7.75{\scriptsize $\pm$0.060}} & \textbf{0.458{\scriptsize $\pm$0.008}} & \textbf{0.462{\scriptsize $\pm$0.006}} & \textbf{0.183{\scriptsize $\pm$0.006}} & \textbf{0.031{\scriptsize $\pm$0.000}} & \textbf{0.156{\scriptsize $\pm$0.004}} \\ \hline
\end{tabular}
\end{table*}
\begin{table*}[th]
\centering
\small
\caption{\small Forecasting future event trajectories.}
\label{tab:future-segment}
\small
\begin{tabular}{c|cc|cc|cc}
\hline
& \multicolumn{2}{c|}{Earthquake} & \multicolumn{2}{c|}{COVID-19} & \multicolumn{2}{c}{Citibike} \\ \cline{2-7}
Model & Spatial $\downarrow$ & Temporal $\downarrow$ & Spatial $\downarrow$ & Temporal $\downarrow$ & Spatial $\downarrow$ & Temporal $\downarrow$ \\ \hline

\multicolumn{7}{c}{Trajectory length = 5} \\ \hline
DSTPP      & 10.8{\scriptsize $\pm$0.067} & 0.573{\scriptsize $\pm$0.013} & 0.612{\scriptsize $\pm$0.006} & 0.070{\scriptsize $\pm$0.002} & 0.045{\scriptsize $\pm$0.000} & 1.054{\scriptsize $\pm$0.007} \\
ARCH-AR   & 8.74{\scriptsize $\pm$0.039} & 0.420{\scriptsize $\pm$0.005} & \textbf{0.446{\scriptsize $\pm$0.004}} & \textbf{0.064{\scriptsize $\pm$0.003}} & \textbf{0.032{\scriptsize $\pm$0.000}} & 0.935{\scriptsize $\pm$0.005} \\  
ARCH          & \textbf{8.73{\scriptsize $\pm$0.042}} & \textbf{0.416{\scriptsize $\pm$0.002}} & 0.448{\scriptsize $\pm$0.004} & 0.071{\scriptsize $\pm$0.004} & \textbf{0.032{\scriptsize $\pm$0.000}} & \textbf{0.910{\scriptsize $\pm$0.008}} \\ \hline

\multicolumn{7}{c}{Trajectory length = 10} \\ \hline
DSTPP     & 11.1{\scriptsize $\pm$0.085} & 0.589{\scriptsize $\pm$0.014} & 0.626{\scriptsize $\pm$0.005} & 0.083{\scriptsize $\pm$0.002} & 0.045{\scriptsize $\pm$0.000} & 0.844{\scriptsize $\pm$0.006} \\
ARCH-AR   & 8.76{\scriptsize $\pm$0.054} & \textbf{0.463{\scriptsize $\pm$0.002}} & \textbf{0.459{\scriptsize $\pm$0.004}} & \textbf{0.068{\scriptsize $\pm$0.002}} & \textbf{0.032{\scriptsize $\pm$0.000}} & 0.695{\scriptsize $\pm$0.004} \\ 
ARCH          & \textbf{8.74{\scriptsize $\pm$0.054}} & 0.465{\scriptsize $\pm$0.006} & 0.461{\scriptsize $\pm$0.003} & 0.087{\scriptsize $\pm$0.005} & \textbf{0.032{\scriptsize $\pm$0.000}} & \textbf{0.626{\scriptsize $\pm$0.008}} \\ \hline
\end{tabular}
\vspace{-0.2in}
\end{table*}

\paragraph{Event Trajectory Recovery.}
Fourth, we evaluated \ours on recovering a missing event trajectory, i.e., a contiguous subsequence within an event sequence. For each test sequence, we randomly masked out a subsequence and applied \ours to predict the masked events conditioned on the remaining observations. We considered masked subsequences of length $5$ and $10$. As shown in Table~\ref{tab:partial-recovery}, \ours consistently outperforms KNNImpute in most cases. The only exception occurs when the masked trajectory length is $5$, where the temporal prediction error of \ours is slightly higher than that of KNNImpute with $k=10$. Notably, when the masked trajectory length increases to $10$, \ours consistently outperforms KNNImpute across all datasets and metrics. This suggests that \ours is particularly effective at capturing longer-range dependencies and reconstructing coherent event trajectories from partial observations.

\paragraph{Forecasting Future Trajectories.}
Finally, we evaluated \ours on jointly forecasting a trajectory of future events. We ran DSTPP for comparison. Since DSTPP is restricted to auto-regressive prediction, each future event is  generated sequentially conditioned on the observed history and previously generated events. We also evaluated an autoregressive variant of our method, denoted by \ours-AR, which forecasts future events one at a time using the same autoregressive procedure. 
As shown in Table~\ref{tab:future-segment}, both \ours-AR and \ours consistently outperform DSTPP, often by a large margin. The only exception is temporal prediction on COVID-19, where the joint forecasting variant of \ours is slightly worse than DSTPP. However, when using the same autoregressive forecasting strategy, \ours-AR still outperforms DSTPP. Together, these results demonstrate the stronger long-horizon forecasting capability of our method. Overall, \ours achieves better or comparable accuracy relative to \ours-AR, showing that joint forecasting can generate multiple future events more efficiently while retaining accuracy comparable to the standard autoregressive forecasting strategy used by existing methods.

\cmt{
\begin{table}[t]
\centering
\small
\setlength{\tabcolsep}{6pt}
\renewcommand{\arraystretch}{1.05}
\caption{\small RMSE error of the conditional intensity on synthetic datasets. Best results are shown in bold.}
\label{tab:synthetic-rmse}
\begin{tabular}{lcc}
\toprule
Model & \textbf{STHP-SYN} & \textbf{STSC-SYN} \\
\midrule
AutoSTPP & $3.70 \pm 2.88$ & $0.397 \pm 0.216$\\
DeepSTPP & $3.89 \pm 2.82$ & $0.631 \pm 0.315$ \\
NSTPP    & $5.58 \pm 3.21$ & $0.578 \pm 0.318$ \\
ARCH     & $\mathbf{2.44 \pm 1.57}$ & $\mathbf{0.248 \pm 0.142}$ \\
\bottomrule
\end{tabular}
\end{table}

}

\vspace{-0.05in}
\section{Conclusion}
\vspace{-0.05in}
We presented \ours, a hierarchical flow-matching framework for spatiotemporal event modeling. By decomposing event generation into conditional temporal and spatial flows, \ours combines the flexibility of continuous generative modeling with tractable and accurate evaluation of conditional intensities and instantaneous event risks. Built on an encoder--decoder architecture with flexible masking strategies, \ours supports a broad range of inference and prediction tasks, including inverse inference, missing-event and missing-attribute completion, event-trajectory recovery, and future-trajectory forecasting. This substantially broadens the scope of spatiotemporal event modeling beyond the standard autoregressive next-event prediction setting targeted by most existing methods. Future work will investigate broader real-world applications and extend the framework toward scalable general-purpose models for spatiotemporal event data.

\bibliography{ref.bib}

@article{lipman2022flow,
  title={Flow matching for generative modeling},
  author={Lipman, Yaron and Chen, Ricky TQ and Ben-Hamu, Heli and Nickel, Maximilian and Le, Matt},
  journal={arXiv preprint arXiv:2210.02747},
  year={2022}
}

@inproceedings{yuan2023DSTPP,
author = {Yuan, Yuan and Ding, Jingtao and Shao, Chenyang and Jin, Depeng and Li, Yong},
title = {Spatio-temporal Diffusion Point Processes},
year = {2023},
isbn = {9798400701030},
publisher = {Association for Computing Machinery},
address = {New York, NY, USA},
url = {https://doi.org/10.1145/3580305.3599511},
doi = {10.1145/3580305.3599511},
booktitle = {Proceedings of the 29th ACM SIGKDD Conference on Knowledge Discovery and Data Mining},
pages = {3173–3184},
numpages = {12},
location = {Long Beach, CA, USA},
series = {KDD '23}
}

@article{sill1997monotonic,
  title={Monotonic networks},
  author={Sill, Joseph},
  journal={Advances in neural information processing systems},
  volume={10},
  year={1997}
}

@article{chen2018neural,
  title={Neural ordinary differential equations},
  author={Chen, Ricky TQ and Rubanova, Yulia and Bettencourt, Jesse and Duvenaud, David K},
  journal={Advances in neural information processing systems},
  volume={31},
  year={2018}
}

@inproceedings{chen2021neuralstpp,
title={Neural Spatio-Temporal Point Processes},
author={Ricky T. Q. Chen and Brandon Amos and Maximilian Nickel},
booktitle={International Conference on Learning Representations},
year={2021},
}

@InProceedings{pmlr-v168-zhou22a,
  title = 	 {Neural Point Process for Learning Spatiotemporal Event Dynamics},
  author =       {Zhou, Zihao and Yang, Xingyi and Rossi, Ryan and Zhao, Handong and Yu, Rose},
  booktitle = 	 {Proceedings of The 4th Annual Learning for Dynamics and Control Conference},
  pages = 	 {777--789},
  year = 	 {2022},
  editor = 	 {Firoozi, Roya and Mehr, Negar and Yel, Esen and Antonova, Rika and Bohg, Jeannette and Schwager, Mac and Kochenderfer, Mykel},
  volume = 	 {168},
  series = 	 {Proceedings of Machine Learning Research},
  month = 	 {23--24 Jun},
  publisher =    {PMLR},
  url = 	 {https://proceedings.mlr.press/v168/zhou22a.html}
}

@inproceedings{NEURIPS2023_9d30c2de,
 author = {Zhou, Zihao and Yu, Rose},
 booktitle = {Advances in Neural Information Processing Systems},
 editor = {A. Oh and T. Naumann and A. Globerson and K. Saenko and M. Hardt and S. Levine},
 pages = {50237--50253},
 publisher = {Curran Associates, Inc.},
 title = {Automatic Integration for Spatiotemporal Neural Point Processes},
 url = {https://proceedings.neurips.cc/paper_files/paper/2023/file/9d30c2def27b5c6a5fb21a9aa5c16f8f-Paper-Conference.pdf},
 volume = {36},
 year = {2023}
}

@article{ho2020denoising,
  title={Denoising diffusion probabilistic models},
  author={Ho, Jonathan and Jain, Ajay and Abbeel, Pieter},
  journal={Advances in neural information processing systems},
  volume={33},
  pages={6840--6851},
  year={2020}
}

@article{dormand1980family,
  title={A family of embedded {Runge--Kutta} formulae},
  author={Dormand, John R. and Prince, Peter J.},
  journal={Journal of Computational and Applied Mathematics},
  volume={6},
  number={1},
  pages={19--26},
  year={1980},
  publisher={Elsevier}
}

@article{scikit-learn,
  title={Scikit-learn: Machine Learning in {P}ython},
  author={Pedregosa, Fabian and Varoquaux, Ga{\"e}l and Gramfort, Alexandre and Michel, Vincent and Thirion, Bertrand and Grisel, Olivier and Blondel, Mathieu and Prettenhofer, Peter and Weiss, Ron and Dubourg, Vincent and Vanderplas, Jake and Passos, Alexandre and Cournapeau, David and Brucher, Matthieu and Perrot, Matthieu and Duchesnay, {\'E}douard},
  journal={Journal of Machine Learning Research},
  volume={12},
  pages={2825--2830},
  year={2011}
}

@inproceedings{du2016recurrent,
	title={Recurrent marked temporal point processes: Embedding event history to vector},
	author={Du, Nan and Dai, Hanjun and Trivedi, Rakshit and Upadhyay, Utkarsh and Gomez-Rodriguez, Manuel and Song, Le},
	booktitle={Proceedings of the 22nd ACM SIGKDD International Conference on Knowledge Discovery and Data Mining},
	pages={1555--1564},
	year={2016}
}

@inproceedings{lloyd2015variational,
	title={Variational inference for Gaussian process modulated Poisson processes},
	author={Lloyd, Chris and Gunter, Tom and Osborne, Michael and Roberts, Stephen},
	booktitle={International Conference on Machine Learning},
	pages={1814--1822},
	year={2015},
	organization={PMLR}
}

@inproceedings{zuo2020transformer,
	title={Transformer hawkes process},
	author={Zuo, Simiao and Jiang, Haoming and Li, Zichong and Zhao, Tuo and Zha, Hongyuan},
	booktitle={International Conference on Machine Learning},
	pages={11692--11702},
	year={2020},
	organization={PMLR}
}

@article{daley2003introduction,
	title={An introduction to the theory of point processes, volume 1: Elementary theory and methods},
	author={Daley, Daryl J and Vere-Jones, David},
	journal={Verlag New York Berlin Heidelberg: Springer},
	year={2003}
}

@inproceedings{zhang2020self,
	title={Self-attentive hawkes process},
	author={Zhang, Qiang and Lipani, Aldo and Kirnap, Omer and Yilmaz, Emine},
	booktitle={International Conference on Machine Learning},
	pages={11183--11193},
	year={2020},
	organization={PMLR}
}

@inproceedings{mei2017neural,
	title={The neural hawkes process: A neurally self-modulating multivariate point process},
	author={Mei, Hongyuan and Eisner, Jason M},
	booktitle={Advances in Neural Information Processing Systems},
	pages={6754--6764},
	year={2017}
}

@inproceedings{zhe2018stochastic,
	title={Stochastic Nonparametric Event-Tensor Decomposition},
	author={Zhe, Shandian and Du, Yishuai},
	booktitle={Advances in Neural Information Processing Systems},
	pages={6856--6866},
	year={2018}
}

@article{ogata1981lewis,
  title={On {L}ewis' simulation method for point processes},
  author={Ogata, Yosihiko},
  journal={IEEE transactions on information theory},
  volume={27},
  number={1},
  pages={23--31},
  year={1981},
  publisher={IEEE}
}

@article{hochreiter1997long,
	title={Long short-term memory},
	author={Hochreiter, Sepp and Schmidhuber, J{\"u}rgen},
	journal={Neural computation},
	volume={9},
	number={8},
	pages={1735--1780},
	year={1997},
	publisher={MIT Press}
}

@article{jia2019neural,
  title={Neural jump stochastic differential equations},
  author={Jia, Junteng and Benson, Austin R},
  journal={Advances in Neural Information Processing Systems},
  volume={32},
  year={2019}
}

@inproceedings{blundell2012modelling,
  title={Modelling reciprocating relationships with Hawkes processes},
  author={Blundell, Charles and Beck, Jeff and Heller, Katherine A},
  booktitle={Advances in Neural Information Processing Systems},
  pages={2600--2608},
  year={2012}
}

@article{hawkes1971spectra,
  title={Spectra of some self-exciting and mutually exciting point processes},
  author={Hawkes, Alan G},
  journal={Biometrika},
  volume={58},
  number={1},
  pages={83--90},
  year={1971},
  publisher={Oxford University Press}
}

@inproceedings{wang2017predicting,
  title={Predicting user activity level in point processes with mass transport equation},
  author={Wang, Yichen and Ye, Xiaojing and Zha, Hongyuan and Song, Le},
  booktitle={Advances in Neural Information Processing Systems},
  pages={1644--1654},
  year={2017}
}

@article{cox1972regression,
  title={Regression models and life-tables},
  author={Cox, David R},
  journal={Journal of the royal statistical society: Series B (methodological)},
  volume={34},
  number={2},
  pages={187--202},
  year={1972},
  publisher={Wiley Online Library}
}

@article{troyanskaya2001missing,
  title={Missing value estimation methods for {DNA} microarrays},
  author={Troyanskaya, Olga and Cantor, Michael and Sherlock, Gavin and Brown, Pat and Hastie, Trevor and Tibshirani, Robert and Botstein, David and Altman, Russ B},
  journal={Bioinformatics},
  volume={17},
  number={6},
  pages={520--525},
  year={2001},
  publisher={Oxford University Press}
}

@inproceedings{10.1145/3580305.3599511,
author = {Yuan, Yuan and Ding, Jingtao and Shao, Chenyang and Jin, Depeng and Li, Yong},
title = {Spatio-temporal Diffusion Point Processes},
year = {2023},
isbn = {9798400701030},
publisher = {Association for Computing Machinery},
address = {New York, NY, USA},
url = {https://doi.org/10.1145/3580305.3599511},
doi = {10.1145/3580305.3599511},
booktitle = {Proceedings of the 29th ACM SIGKDD Conference on Knowledge Discovery and Data Mining},
pages = {3173–3184},
numpages = {12},
location = {Long Beach, CA, USA},
series = {KDD '23}
}

@article{braga2001effects,
  title={The effects of hot spots policing on crime},
  author={Braga, Anthony A},
  journal={The ANNALS of the American Academy of Political and Social Science},
  volume={578},
  number={1},
  pages={104--125},
  year={2001},
  publisher={Sage Publications Sage CA: Thousand Oaks, CA}
}

@article{
lin2022exploring,
title={Exploring Generative Neural Temporal Point Process},
author={Haitao Lin and Lirong Wu and Guojiang Zhao and Liu Pai and Stan Z. Li},
journal={Transactions on Machine Learning Research},
issn={2835-8856},
year={2022},
url={https://openreview.net/forum?id=NPfS5N3jbL},
note={}
}

@article{ludke2023add,
  title={Add and thin: Diffusion for temporal point processes},
  author={L{\"u}dke, David and Bilo{\v{s}}, Marin and Shchur, Oleksandr and Lienen, Marten and G{\"u}nnemann, Stephan},
  journal={Advances in Neural Information Processing Systems},
  volume={36},
  pages={56784--56801},
  year={2023}
}

@inproceedings{
ludke2026editbased,
title={Edit-Based Flow Matching for Temporal Point Processes},
author={David L{\"u}dke and Marten Lienen and Marcel Kollovieh and Stephan G{\"u}nnemann},
booktitle={The Fourteenth International Conference on Learning Representations},
year={2026},
url={https://openreview.net/forum?id=FNf9IV1P2L}
}

@inproceedings{
kerrigan2026eventflow,
title={Event{F}low: Forecasting Temporal Point Processes with Flow Matching},
author={Gavin Kerrigan and Kai Nelson and Padhraic Smyth},
booktitle={The 29th International Conference on Artificial Intelligence and Statistics},
year={2026},
url={https://openreview.net/forum?id=QXqKGOE2JW}
}

@book{coddington1955theory,
  title={Theory of Ordinary Differential Equations},
  author={Coddington, Earl A. and Levinson, Norman},
  publisher={McGraw-Hill},
  year={1955}
}

@article{
    shchur2020intensity,
    title={Intensity-Free Learning of Temporal Point Processes},
    author={Oleksandr Shchur and Marin Bilo\v{s} and Stephan G\"{u}nnemann},
    journal={International Conference on Learning Representations (ICLR)},
    year={2020},
}

@inproceedings{wang2022nonparametric,
  title={Nonparametric embeddings of sparse high-order interaction events},
  author={Wang, Zheng and Xu, Yiming and Tillinghast, Conor and Li, Shibo and Narayan, Akil and Zhe, Shandian},
  booktitle={International Conference on Machine Learning},
  pages={23237--23253},
  year={2022},
  organization={PMLR}
}

@article{loshchilov2017decoupled,
  title={Decoupled weight decay regularization},
  author={Loshchilov, Ilya and Hutter, Frank},
  journal={arXiv preprint arXiv:1711.05101},
  year={2017}
}

@article{yuan2019multivariate,
  title={Multivariate spatiotemporal {H}awkes processes and network reconstruction},
  author={Yuan, Baichuan and Li, Hao and Bertozzi, Andrea L and Brantingham, P Jeffrey and Porter, Mason A},
  journal={SIAM Journal on Mathematics of Data Science},
  volume={1},
  number={2},
  pages={356--382},
  year={2019},
  publisher={SIAM}
}

@article{reinhart2018review,
  title={A review of self-exciting spatio-temporal point processes and their applications},
  author={Reinhart, Alex},
  journal={Statistical Science},
  volume={33},
  number={3},
  pages={299--318},
  year={2018},
  publisher={JSTOR}
}
\bibliographystyle{apalike}








\newpage
\appendix
\onecolumn
\section*{Appendix}
\section{Proof of Lemma~\ref{lem}}\label{sect:appendix:proof}

\begin{proof}
Let \(\Phi_t:\mathbb{R}\to\mathbb{R}\) denote the flow map induced by the ODE, so that for each initial value \(s_0\in\mathbb{R}\), the corresponding solution at time \(t\) is
\[
\Phi_t(s_0).
\]
Thus, if \(S(0)\) is the random initial state, then
\[
S(t)=\Phi_t(S(0)),
\]
and for the deterministic trajectory \(\{s(t)\}_{t\ge 0}\) with initial value \(s(0)\), we have
\[
s(t)=\Phi_t(s(0)).
\]

Under the standard continuity and Lipschitz assumptions of the Picard--Lindel\"of theorem, the ODE admits a unique solution for each initial condition~\citep{coddington1955theory}.
In one dimension, uniqueness implies that the flow map \(\Phi_t\) is order-preserving: if \(s_0^{(1)} < s_0^{(2)}\), then
\[
\Phi_t\bigl(s_0^{(1)}\bigr) < \Phi_t\bigl(s_0^{(2)}\bigr), \qquad \forall t \ge 0,
\]
since otherwise the two trajectories would intersect at some time, contradicting uniqueness.

Because \(\Phi_t\) is strictly increasing, we have
\[
S(t)\le s(t)
\quad\Longleftrightarrow\quad
\Phi_t(S(0)) \le \Phi_t(s(0))
\quad\Longleftrightarrow\quad
S(0)\le s(0).
\]
Taking probabilities on both sides yields
\[
F_t(s(t))
=
\mathbb{P}\bigl(S(t)\le s(t)\bigr)
=
\mathbb{P}\bigl(S(0)\le s(0)\bigr)
=
F_0(s(0)).
\]
Therefore, the CDF value remains constant along every solution trajectory.
\end{proof}

\section{Dataset Details}\label{sect:dataset}
\paragraph{Earthquake}\citep{chen2021neuralstpp}.
This dataset records earthquake and aftershock occurrences in Japan between Year 1990 and 2020. Each sequence aggregates one month of events, so the prediction horizon is $T=30$ days. The processed benchmark contains $1{,}050$ sequences in total, split into $950$ training, $50$ validation, and $50$ test sequences. According to the provided data files, sequence lengths span from $22$ to $554$ events, with an average of $87.5$ events per sequence.

\paragraph{COVID-19.}~\citep{chen2021neuralstpp}
The  dataset stores daily COVID-19 case events across counties in New Jersey from March 2020 to July 2020. One sequence corresponds to a one-week window, i.e., $T=7$ days. The full dataset contains $1{,}650$ sequences, partitioned into $1{,}450$ for training, $100$ for validation, and $100$ for testing. Sequence lengths vary from $5$ to $287$, and the average length is $97.8$.

\paragraph{Citibike.}~\citep{yuan2023DSTPP}
The {Citibike} benchmark is built from New York City bike-sharing trips collected from April to August 2019. The temporal unit is one hour, and each sequence covers a single day, giving $T=24$. We use $2{,}440$ training sequences, $300$ validation sequences, and $320$ test sequences. The processed files contain $3{,}060$ sequences in total, with lengths ranging from $14$ to $204$ and an average of $134.9$ events.

\begin{figure*}[t]
    \centering
    \includegraphics[width=\textwidth]{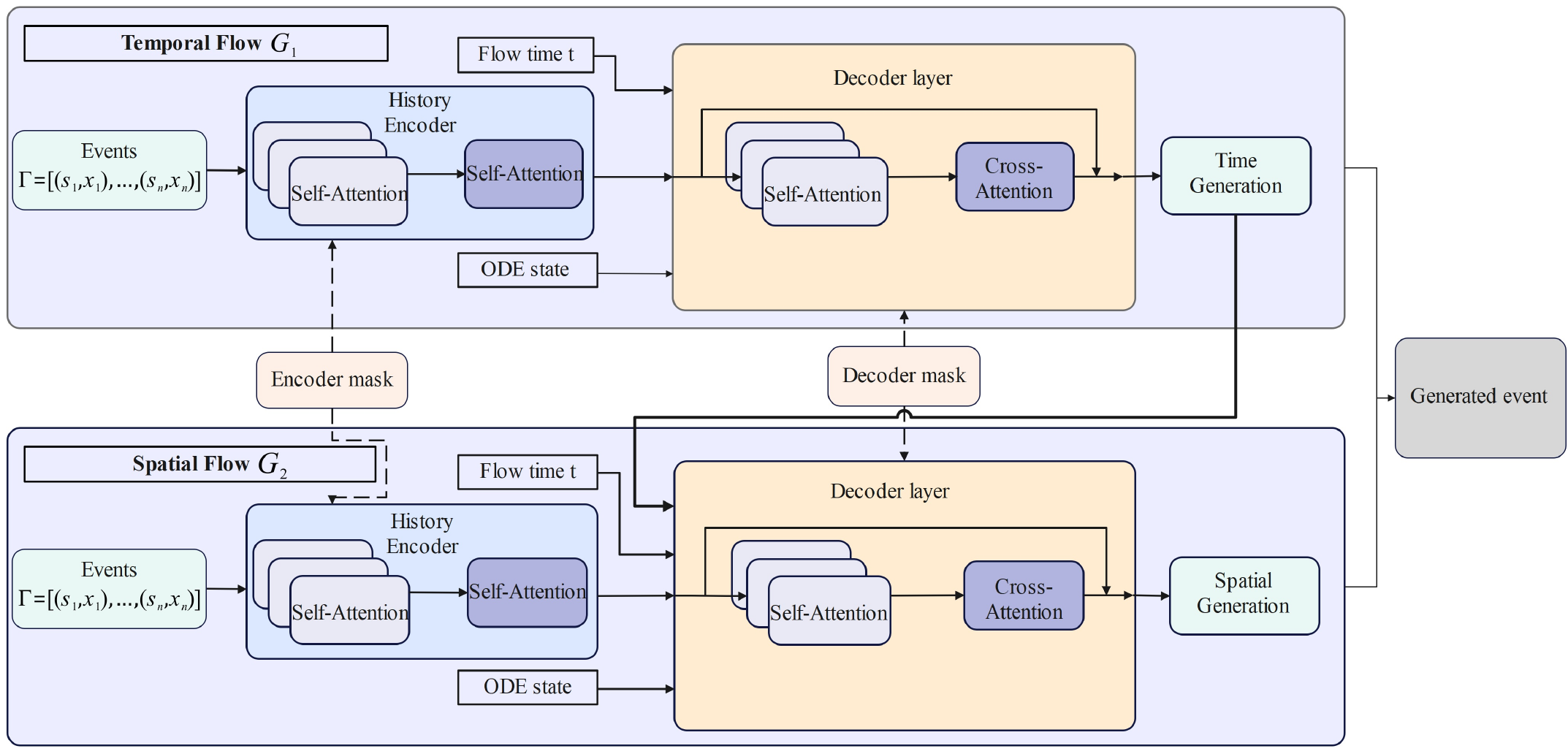}
    \caption{\small Overview of ARCH framework for spatiotemporal event modeling.}
    \label{fig:arch-architecture}
\end{figure*}

\begin{figure*}[t]
    \centering
    \includegraphics[width=\textwidth]{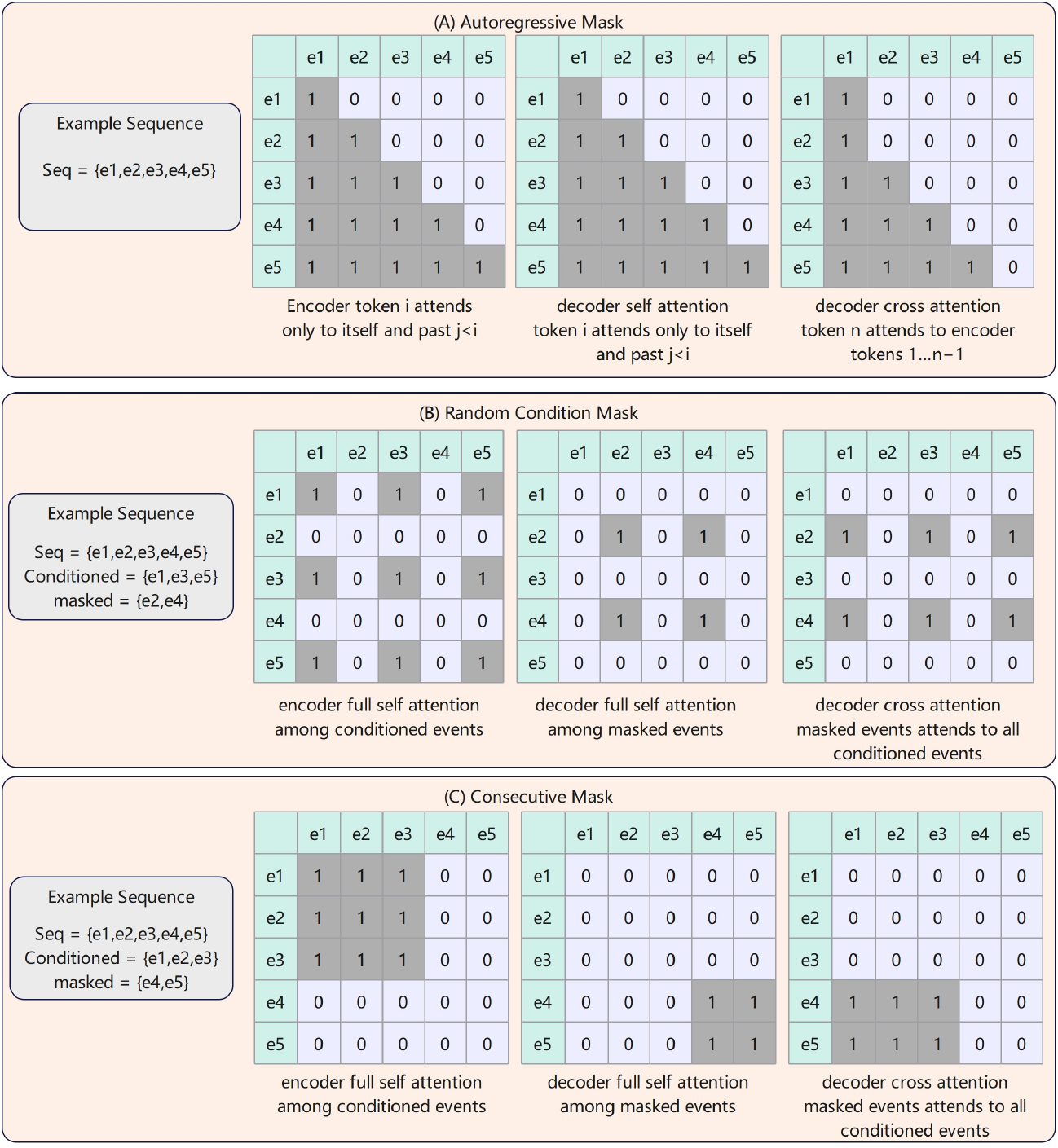}
    \caption{Illustration of the masking strategies used in ARCH training: autoregressive mask, random condition mask, and consecutive mask.}
    \label{fig:mask-strategy}
\end{figure*}



\section{Training Details and Hyperparameter Settings}\label{sect:appendix:hyper}

For all real-world datasets, we apply a linear shift and scaling to normalize the spatial domain to $[0,1]^2$ during training. At evaluation time, both the generated event times and spatial locations are mapped back to their original domains before computing the prediction errors.

In our implementation, the temporal and spatial flows share the same Transformer backbone for the encoder and decoder, but use separate MLP heads for velocity prediction. The inter-event time $\Delta t_i$ is passed through a sinusoidal encoding layer to obtain its latent representation. The flow time $t \in [0,1]$ is first encoded by a sinusoidal encoding layer and then transformed by an MLP. The ODE states and event spatial locations are separately embedded using MLP layers. The relevant representations are then summed to form token embeddings, which are processed by the self-attention layers in the encoder and decoder.

When summing the loss terms, all temporal velocity losses are assigned weight $1$, while all spatial velocity losses are assigned weight $1.5$. The three mask types are treated equally, i.e., no additional weighting is applied across mask types. For the random-conditioning mask, each binary mask indicator is sampled independently, with probability $0.7$ of being one. To sample a consecutive mask, we uniformly sample two endpoints $a$ and $b$ and resample until $a<b$. The architectural hyperparameters of \ours are listed in Table~\ref{tab:hyper-setting}.

All experiments were conducted on a Linux workstation equipped with a NVIDIA GeForce RTX 3090 GPU card (24 GB Memory).

\begin{table}[t]
\centering
\small
\setlength{\tabcolsep}{6pt}
\renewcommand{\arraystretch}{1.05}
\caption{\small Hyperparameter settings of \ours.}
\label{tab:hyper-setting}
\begin{tabular}{cc}
\toprule
Hyperparameter & Setting \\
\midrule
Embedding Dimension in Encoder & 32 \\
Embedding Dimension in Decoder & 32 \\
Sinusoidal Encoding dimension & 32\\
Input Layer for ODE state of the temporal flow  & MLP (64, 32)/GELU \\
Input Layer for ODE state of the spatial flow & MLP (64, 32)/ GELU \\
Input Layer for event spatial locations & MLP (64, 32)/GELU\\
Input Layer for flow time $t$ &  Sinusoidal Encoding $\rightarrow$ MLP [32, 32]/GELU\\
Input Layer for $\Delta s$ & Log Transform $\rightarrow$ Sinusoidal Encoding\\
Number of Encoder layers & 3 \\
Number of Decoder layers & 3 \\
Layer Norm First & True\\
Dropout & 0.15\\
Feed Forward Dimension in Encoder & 32 \\
Feed Forward Dimension in Decoder & 32\\
Number of Attention Heads & 1\\
Temporal Velocity Head & MLP (256, 1)/GELU \\
Spatial Velocity Head & MLP (256, $d$)/GELU \\
Optimizer & AdamW \\
Learning rate & $10^{-3}$ \\
\bottomrule
\end{tabular}
\end{table}

\begin{table*}[th]
\centering
\small
\caption{\small Missing event recovery.}
\label{tb:completion}
\small
\begin{tabular}{c|cc|cc|cc}
\hline
& \multicolumn{2}{c|}{Earthquake} & \multicolumn{2}{c|}{COVID-19} & \multicolumn{2}{c}{Citibike} \\ \cline{2-7}
Model & Spatial $\downarrow$ & Temporal $\downarrow$ & Spatial $\downarrow$ & Temporal $\downarrow$ & Spatial $\downarrow$ & Temporal $\downarrow$ \\ \hline

\multicolumn{7}{c}{Missing ratio = 10\%} \\ \hline
KNNImp ($k=3$)  & 7.28{\scriptsize $\pm$0.096} & 0.414{\scriptsize $\pm$0.012} & 0.462{\scriptsize $\pm$0.003} & 0.096{\scriptsize $\pm$0.003} & 0.037{\scriptsize $\pm$0.000} & 0.228{\scriptsize $\pm$0.006} \\
KNNImp ($k=5$)  & 7.04{\scriptsize $\pm$0.075} & 0.395{\scriptsize $\pm$0.012} & 0.441{\scriptsize $\pm$0.003} & 0.092{\scriptsize $\pm$0.004} & 0.035{\scriptsize $\pm$0.000} & 0.221{\scriptsize $\pm$0.006} \\
KNNImp ($k=10$) & 6.86{\scriptsize $\pm$0.073} & 0.381{\scriptsize $\pm$0.013} & 0.424{\scriptsize $\pm$0.002} & \textbf{0.090{\scriptsize $\pm$0.004}} & 0.034{\scriptsize $\pm$0.000} & 0.216{\scriptsize $\pm$0.006} \\
\ours  & \textbf{6.71{\scriptsize $\pm$0.022}} & \textbf{0.365{\scriptsize $\pm$0.011}} & \textbf{0.379{\scriptsize $\pm$0.003}} & 0.094{\scriptsize $\pm$0.005} & \textbf{0.031{\scriptsize $\pm$0.000}} & \textbf{0.204{\scriptsize $\pm$0.005}} \\ \hline

\multicolumn{7}{c}{Missing ratio = 20\%} \\ \hline
KNNImp ($k=3$)  & 7.20{\scriptsize $\pm$0.072} & 0.418{\scriptsize $\pm$0.007} & 0.469{\scriptsize $\pm$0.003} & 0.102{\scriptsize $\pm$0.002} & 0.037{\scriptsize $\pm$0.000} & 0.226{\scriptsize $\pm$0.004} \\
KNNImp ($k=5$)  & 6.95{\scriptsize $\pm$0.050} & 0.402{\scriptsize $\pm$0.007} & 0.447{\scriptsize $\pm$0.003} & 0.099{\scriptsize $\pm$0.002} & 0.035{\scriptsize $\pm$0.000} & 0.219{\scriptsize $\pm$0.004} \\
KNNImp ($k=10$) & 6.77{\scriptsize $\pm$0.049} & 0.388{\scriptsize $\pm$0.007} & 0.429{\scriptsize $\pm$0.003} & \textbf{0.097{\scriptsize $\pm$0.002}} & 0.034{\scriptsize $\pm$0.000} & 0.215{\scriptsize $\pm$0.004} \\
\ours  & \textbf{6.57{\scriptsize $\pm$0.050}} & \textbf{0.372{\scriptsize $\pm$0.007}} & \textbf{0.386{\scriptsize $\pm$0.003}} & 0.103{\scriptsize $\pm$0.002} & \textbf{0.031{\scriptsize $\pm$0.000}} & \textbf{0.198{\scriptsize $\pm$0.003}} \\ \hline

\multicolumn{7}{c}{Missing ratio = 30\%} \\ \hline
KNNImp ($k=3$)  & 7.20{\scriptsize $\pm$0.028} & 0.416{\scriptsize $\pm$0.003} & 0.469{\scriptsize $\pm$0.002} & 0.101{\scriptsize $\pm$0.001} & 0.037{\scriptsize $\pm$0.000} & 0.228{\scriptsize $\pm$0.002} \\
KNNImp ($k=5$)  & 6.94{\scriptsize $\pm$0.026} & 0.398{\scriptsize $\pm$0.003} & 0.447{\scriptsize $\pm$0.002} & 0.098{\scriptsize $\pm$0.001} & 0.036{\scriptsize $\pm$0.000} & 0.220{\scriptsize $\pm$0.002} \\
KNNImp ($k=10$) & 6.76{\scriptsize $\pm$0.022} & 0.384{\scriptsize $\pm$0.004} & 0.430{\scriptsize $\pm$0.001} & \textbf{0.095{\scriptsize $\pm$0.001}} & 0.034{\scriptsize $\pm$0.000} & 0.215{\scriptsize $\pm$0.002} \\
\ours & \textbf{6.53{\scriptsize $\pm$0.046}} & \textbf{0.366{\scriptsize $\pm$0.006}} & \textbf{0.385{\scriptsize $\pm$0.001}} & 0.099{\scriptsize $\pm$0.001} & \textbf{0.031{\scriptsize $\pm$0.000}} & \textbf{0.196{\scriptsize $\pm$0.002}} \\ \hline
\end{tabular}
\end{table*}

\section{Additional Results for Missing Event Recovery}\label{sect:missing}
We randomly selected a portion of events. Each selected event has one third chance to miss its time, one third chance to miss its location, and one third chance to be completely missing. We then applied \ours to predict the missing part of these events. We considered missing ratios of 5\% and 10\%. We compared with KNNImpute where $k=10$.  As shown in Table~\ref{tab:partial-missing}, \ours achieves better recover accuracy in all the cases except that on COVID-19 dataset, the event time accuracy of \ours is slightly worse than KNNImpute under missing ratio 10\%.

\begin{table*}[th]
\centering
\small
\caption{\small Missing event attribute (time and/or location) recovery.}
\label{tab:partial-missing}
\small
\begin{tabular}{c|cc|cc|cc}
\hline
& \multicolumn{2}{c|}{Earthquake} & \multicolumn{2}{c|}{COVID-19} & \multicolumn{2}{c}{Citibike} \\ \cline{2-7}
Model & Spatial $\downarrow$ & Temporal $\downarrow$ & Spatial $\downarrow$ & Temporal $\downarrow$ & Spatial $\downarrow$ & Temporal $\downarrow$ \\ \hline

\multicolumn{7}{c}{Missing ratio  = 5\%} \\ \hline
KNNImp ($k=10$) & 6.70{\scriptsize $\pm$0.094} & 0.374{\scriptsize $\pm$0.035} & 0.420{\scriptsize $\pm$0.003} & \textbf{0.090{\scriptsize $\pm$0.010}} & 0.034{\scriptsize $\pm$0.000} & 0.221{\scriptsize $\pm$0.009} \\ 
\ours (Ours) & \textbf{6.45{\scriptsize $\pm$0.163}} & \textbf{0.348{\scriptsize $\pm$0.030}} & \textbf{0.383{\scriptsize $\pm$0.007}} & \textbf{0.090{\scriptsize $\pm$0.011}} & \textbf{0.031{\scriptsize $\pm$0.000}} & \textbf{0.214{\scriptsize $\pm$0.009}} \\ \hline

\multicolumn{7}{c}{Missing ratio = 10\%} \\ \hline
KNNImp ($k=10$) & 6.51{\scriptsize $\pm$0.051} & 0.389{\scriptsize $\pm$0.014} & 0.431{\scriptsize $\pm$0.004} & \textbf{0.089{\scriptsize $\pm$0.007}} & 0.034{\scriptsize $\pm$0.000} & 0.220{\scriptsize $\pm$0.005} \\
\ours (Ours) & \textbf{6.26{\scriptsize $\pm$0.140}} & \textbf{0.367{\scriptsize $\pm$0.017}} & \textbf{0.397{\scriptsize $\pm$0.003}} & 0.093{\scriptsize $\pm$0.008} & \textbf{0.032{\scriptsize $\pm$0.000}} & \textbf{0.211{\scriptsize $\pm$0.006}} \\ \hline
\end{tabular}
\end{table*}

\begin{figure*}[t]
    \centering
    \includegraphics[width=\textwidth]{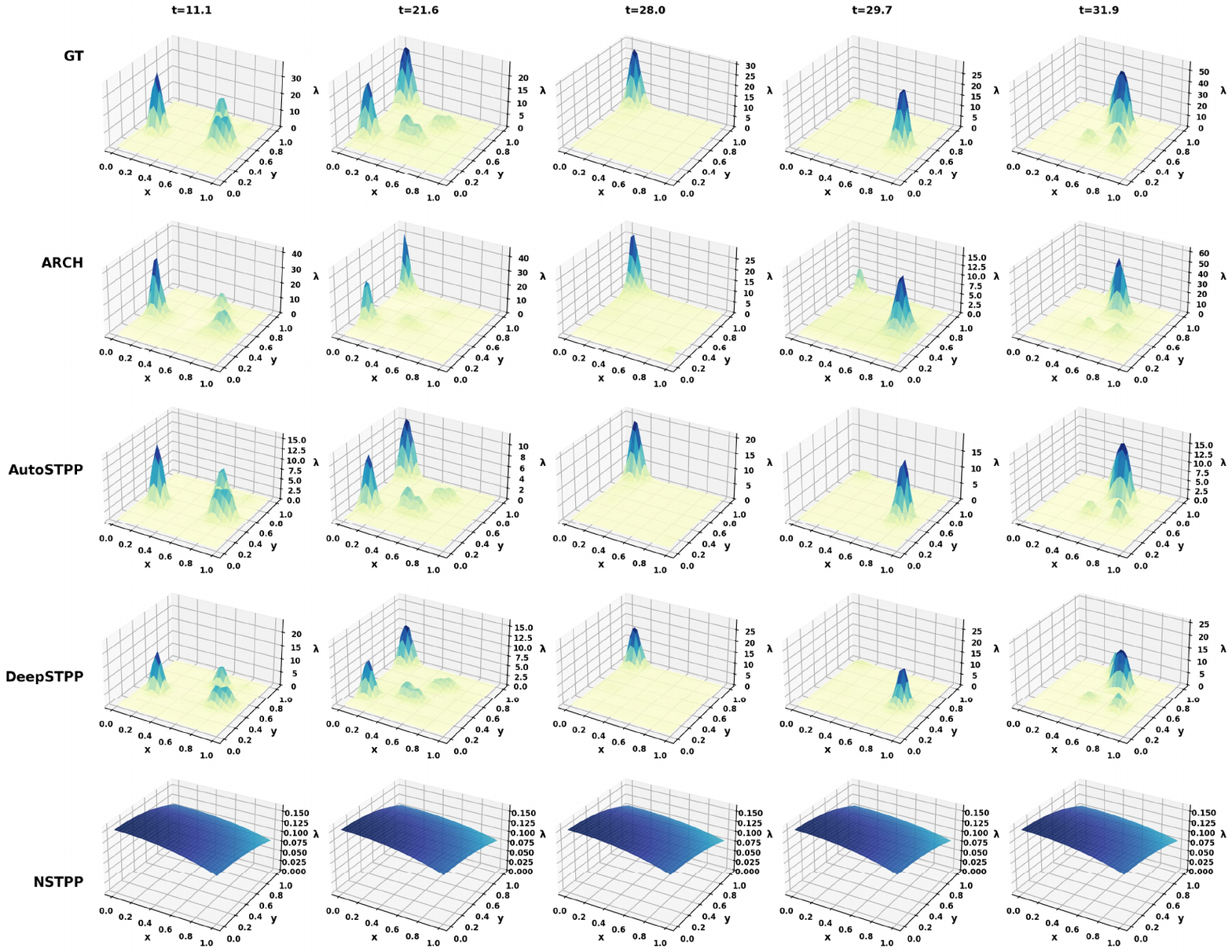}
\caption{\small Conditional intensity functions, i.e., instantaneous event risks, at representative time points of a test sequence simulated by a spatiotemporal Hawkes process (\textbf{SYN1}). GT denotes the ground truth.}
    \label{fig:cif-HP}
\end{figure*}

\begin{algorithm}[t]
\caption{\ours Training}
\label{alg:arch-training}
\begin{algorithmic}[1]
\Require Training dataset $\mathcal{D}$, temporal flow $\v_{\Theta_1}$, spatial flow $\v_{\Theta_2}$
\State Apply a log transformation to each inter-event time: 
$\Delta s_i \mapsto \log(\Delta s_i+\epsilon)$, where $\Delta s_i=s_i-s_{i-1}$ and $\epsilon>0$ is a small constant.
\While{not converged}
    \State Sample a mini-batch $\mathcal{B}$ from $\mathcal{D}$ and pad all sequences to the same length.
    \State For each sequence in $\mathcal{B}$, construct the autoregressive mask $\Mcal_{\mathrm{ar}}$, and sample a random-conditioning mask $\Mcal_{\mathrm{rc}}$ and a consecutive mask $\Mcal_{\mathrm{con}}$.
    \For{each mask type $\Mcal \in \{\Mcal_{\mathrm{ar}}, \Mcal_{\mathrm{rc}}, \Mcal_{\mathrm{con}}\}$}
        \State Combine $\Mcal$ with the padding mask to derive the encoder mask $\Mcal^{\mathrm{enc}}$ and decoder mask $\Mcal^{\mathrm{dec}}$ for $\mathcal{B}$.
        \State Sample initial noise $\Scal_0$ and $\Xcal_0$ from standard normal distributions for $\v_{\Theta_1}$ and $\v_{\Theta_2}$, respectively.
        \State Sample flow times $\tau_1,\tau_2 \sim \mathrm{Uniform}(0,1)$ for the temporal and spatial flows, respectively.
        \State Construct the interpolated flow states:
        \begin{align}
            \Scal_{\tau_1} = (1-\tau_1)\Scal_0 + \tau_1\Scal_1, 
            \qquad
            \Xcal_{\tau_2} = (1-\tau_2)\Xcal_0 + \tau_2\Xcal_1. 
            \notag
        \end{align}
        \State Apply $(\Mcal^{\mathrm{enc}}, \Mcal^{\mathrm{dec}})$ to $\v_{\Theta_1}$ and $\v_{\Theta_2}$ to obtain velocity predictions $\Vcal_1$ and $\Vcal_2$.
        \State Compute the masked flow-matching losses:
        \begin{align}
            \Lcal_1^{\Mcal} 
            &= 
            \left\|
            \Mcal^{\mathrm{dec}} \circ 
            \left(\Vcal_1 - (\Scal_1-\Scal_0)\right)
            \right\|^2,
            \qquad
            \Lcal_2^{\Mcal} 
            = 
            \left\|
            \Mcal^{\mathrm{dec}} \circ 
            \left(\Vcal_2 - (\Xcal_1-\Xcal_0)\right)
            \right\|^2.
            \notag
        \end{align}
    \EndFor
    \State Sum all loss terms, backpropagate the gradients, and update the model parameters.
\EndWhile
\end{algorithmic}
\end{algorithm}

\begin{algorithm}[t]
\caption{\ours Generation}
\label{alg:arch-generation}
\begin{algorithmic}[1]
\Require Observed events, target-event positions, trained temporal flow $\v_{\Theta_1}$ and spatial flow $\v_{\Theta_2}$
\State Construct the encoder mask $\Mcal^{\mathrm{enc}}$ and decoder mask $\Mcal^{\mathrm{dec}}$ according to the observed and target events.
\State Sample initial noise $\Scal_0$ from a standard normal distribution.
\State Starting from $\Scal_0$, integrate the temporal ODE from flow time $0$ to $1$; at each integration step, apply $(\Mcal^{\mathrm{enc}},\Mcal^{\mathrm{dec}})$ to $\v_{\Theta_1}$ to obtain the velocity field.
\State Obtain $\Scal_1$ and map it back to the inter-event-time domain by applying the inverse log transformation, $\Delta s = \exp(\Scal_1)-\epsilon$.
\State Sample initial noise $\Xcal_0$ from a standard normal distribution.
\State Condition the spatial flow $\v_{\Theta_2}$ on $\Scal_1$. Starting from $\Xcal_0$, integrate the spatial ODE from flow time $0$ to $1$; at each integration step, apply $(\Mcal^{\mathrm{enc}},\Mcal^{\mathrm{dec}})$ to $\v_{\Theta_2}$ to obtain the velocity field.
\State Obtain $\Xcal_1$ as the generated spatial locations of the target events.
\end{algorithmic}
\end{algorithm}

\section{Limitations and Future Work}

While our framework supports a broad range of inference tasks beyond autoregressive prediction, it currently assumes that the number of target events, or target event attributes, is specified in advance. It does not explicitly infer how many events are missing within an observed time window, or how many future events should occur within a forecasting horizon. Although event counts can in principle be estimated through repeated sampling and rejection, a more principled and efficient approach would be valuable. We leave this direction for future work.

In addition, our results demonstrate that flow matching provides a flexible framework for capturing diverse event distributions and their instantaneous risks. However, in many applications, structural priors may be important for uncovering interpretable or semantically meaningful event relationships. Incorporating such priors into the proposed framework is another promising direction for future work. Finally, we plan to extend our approach to support discrete event-mask generation, enabling the model to infer not only the values of target events but also which types of events or attributes should be generated.
\newpage
\appendix
\onecolumn



\end{document}